\documentclass{article}

\usepackage[final]{nips_2016}

\usepackage{tikz}
\usetikzlibrary{shapes, arrows, calc, positioning,matrix}
\tikzset{
data/.style={circle, draw, text centered, minimum height=3em ,minimum width = .5em, inner sep = 2pt},
empty/.style={circle, text centered, minimum height=3em ,minimum width = .5em, inner sep = 2pt},
}

\usepackage[utf8]{inputenc} 
\usepackage[T1]{fontenc}    
\usepackage{hyperref}       
\usepackage{url}            
\usepackage{booktabs}       
\usepackage{amsfonts}       
\usepackage{nicefrac}       
\usepackage{microtype}      
\usepackage{tikz}
\usetikzlibrary{positioning}
\usetikzlibrary{arrows}
\usepackage{dsfont}

\usepackage{times}
\usepackage{graphicx} 
\usepackage{caption}
\usepackage{wrapfig}

\usepackage{algorithm}
\usepackage{algpseudocode}

\usepackage{hyperref}

\usepackage{amsmath}
\usepackage{amssymb}

\usepackage[font=footnotesize]{caption,subcaption}
\usepackage{sidecap}

\usepackage{xspace}

\usepackage{color}

\usepackage{amsthm}

\newtheorem{theorem}{Theorem}[section]

\usepackage{amsmath,stackengine}
\stackMath
\theoremstyle{definition}
\newtheorem{definition}{Definition}[section]

\newcommand{\pushright}[1]{\ifmeasuring@#1\else\omit\hfill$\displaystyle#1$\fi\ignorespaces}

\usepackage{mathtools}

\newcommand{\data}{\mathcal{D}}

\newcommand{\real}{\mathbb{R}}

\newcommand{\reward}{r}
\newcommand{\policy}{\pi}
\newcommand{\mdp}{\mathcal{M}}
\newcommand{\states}{\mathcal{S}}
\newcommand{\actions}{\mathcal{A}}
\newcommand{\observations}{\mathcal{O}}
\newcommand{\transitions}{T}
\newcommand{\initstate}{d_0}
\newcommand{\freq}{d}
\newcommand{\obsfunc}{E}

\newcommand{\discount}{\gamma}
\newcommand{\behavior}{{\pi_\beta}}
\newcommand{\bellman}{\mathcal{B}}

\newcommand{\qset}{\mathcal{Q}}
\newcommand{\batch}{B}
\newcommand{\qfeat}{\mathbf{f}}
\newcommand{\Qfeat}{\mathbf{F}}

\newcommand{\traj}{\tau}

\newcommand{\proj}{\Pi}

\newcommand{\return}{\mathcal{R}}

\newcommand{\conpen}{\mathcal{C}}

\newcommand{\en}{\mathcal{E}}

\newcommand{\by}{\mathbf{y}}

\newcommand{\bo}{\mathbf{o}}
\newcommand{\bs}{\mathbf{s}}
\newcommand{\ba}{\mathbf{a}}

\newcommand{\kl}{D_\text{KL}}

\newcommand{\ent}{\mathcal{H}}

\newcommand{\E}{\mathbb{E}}

\title{Offline Reinforcement Learning: Tutorial, Review, and Perspectives on Open Problems}

\author{
Sergey Levine$^{1,2}$, Aviral Kumar$^1$, George Tucker$^2$, Justin Fu$^1$\\
$^1$UC Berkeley, $^2$Google Research, Brain Team\\
}

\begin{document}

\maketitle


\begin{abstract}
In this tutorial article, we aim to provide the reader with the conceptual tools needed to get started on research on offline reinforcement learning algorithms: reinforcement learning algorithms that utilize previously collected data, without additional online data collection. Offline reinforcement learning algorithms hold tremendous promise for making it possible to turn large datasets into powerful decision making engines. Effective offline reinforcement learning methods would be able to extract policies with the maximum possible utility out of the available data, thereby allowing automation of a wide range of decision-making domains, from healthcare and education to robotics. However, the limitations of current algorithms make this difficult. We will aim to provide the reader with an understanding of these challenges, particularly in the context of modern deep reinforcement learning methods, and describe some potential solutions that have been explored in recent work to mitigate these challenges, along with recent applications, and a discussion of perspectives on open problems in the field.
\end{abstract}

\section{Introduction}
\label{sec:introduction}

Reinforcement learning provides a mathematical formalism for learning-based control. By utilizing reinforcement learning, we can automatically acquire near-optimal behavioral skills, represented by policies, for optimizing user-specified reward functions. The reward function defines \emph{what} an agent should do, and a reinforcement learning algorithm determines \emph{how} to do it. While the reinforcement learning algorithms have been an active area of research for decades, the introduction of effective high-capacity function approximators -- deep neural networks -- into reinforcement learning, along with effective algorithms for training them, has allowed reinforcement learning methods to attain excellent results along a wide range of domains~\citep{tesauro1994td,hafner2011reinforcement,levine2013guided,mnih2013playing,levine2016end,silver2017mastering,kalashnikov2018qtopt}.

However, the fact that reinforcement learning algorithms provide a fundamentally \emph{online} learning paradigm is also one of the biggest obstacles to their widespread adoption. The process of reinforcement learning involves iteratively collecting experience by interacting with the environment, typically with the latest learned policy, and then using that experience to improve the policy~\citep{sb-irl-98}. In many settings, this sort of online interaction is impractical, either because data collection is expensive (e.g., in robotics, educational agents, or healthcare) and dangerous (e.g., in autonomous driving, or healthcare). Furthermore, even in domains where online interaction is feasible, we might still prefer to utilize previously collected data instead -- for example, if the domain is complex and effective generalization requires large datasets.

Indeed, the success of machine learning methods across a range of practically relevant problems over the past decade can in large part be attributed to the advent of scalable \emph{data-driven} learning methods, which become better and better as they are trained with more data. Online reinforcement learning is difficult to reconcile with this paradigm. While this was arguably less of an issue when reinforcement learning methods utilized low-dimensional or linear parameterizations, and therefore relied on small datasets for small problems that were easy to collect or simulate~\citep{lange2012batch}, once deep networks are incorporated into reinforcement learning, it is tempting to consider whether the same kind of data-driven learning can be applied with reinforcement learning objectives, thus resulting in \emph{data-driven reinforcement learning} that utilizes only previously collected offline data, without any additional online interaction~\citep{kumar_blog,d4rl}. See Figure~\ref{fig:introduction} for a pictorial illustration. A number of recent works have illustrated the power of such an approach in enabling data-driven learning of policies for dialogue~\citep{jaques2019way}, robotic manipulation behaviors~\citep{ebert2018visual,kalashnikov2018qtopt}, and robotic navigation skills~\citep{kahn2020badgr}.

\begin{figure}
    \centering
    \vspace{-5pt}
    \includegraphics[width=0.98\columnwidth]{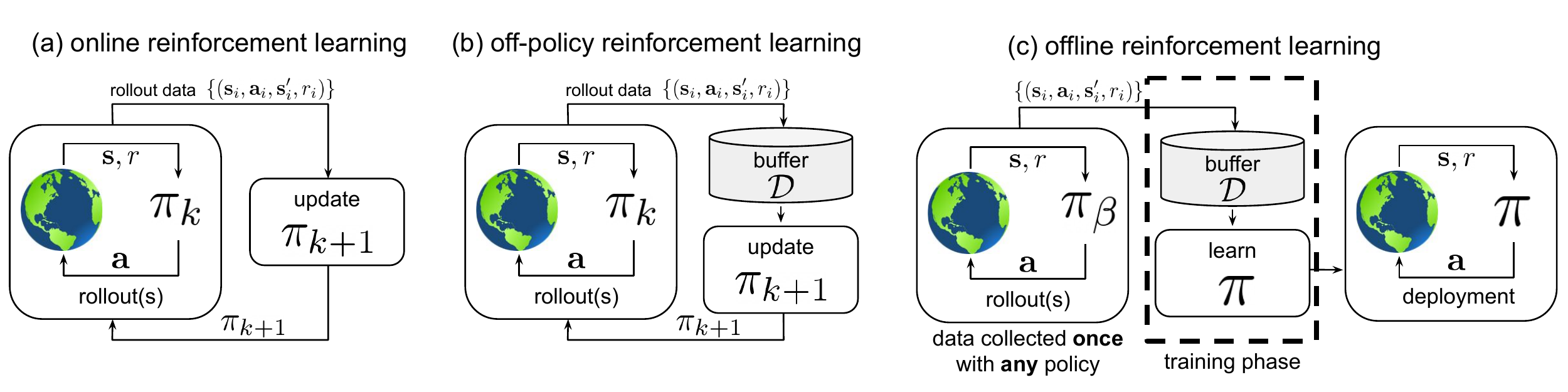}
    \caption{Pictorial illustration of classic online reinforcement learning (a), classic off-policy reinforcement learning (b), and offline reinforcement learning (c). In online reinforcement learning (a), the policy $\policy_k$ is updated with streaming data collected by $\policy_k$ itself. In the classic off-policy setting (b), the agent's experience is appended to a data buffer (also called a replay buffer) $\data$, and each new policy $\policy_k$ collects additional data, such that $\data$ is composed of samples from $\policy_0, \policy_1, \dots, \policy_k$, and all of this data is used to train an updated new policy $\policy_{k+1}$. In contrast, offline reinforcement learning employs a dataset $\data$ collected by some (potentially unknown) behavior policy $\behavior$. The dataset is collected once, and is not altered during training, which makes it feasible to use large previous collected datasets. The training process does not interact with the MDP at all, and the policy is only deployed after being fully trained.}
    \label{fig:introduction}
    \vspace{-10pt}
\end{figure}

Unfortunately, such data-driven \emph{offline} reinforcement learning also poses major algorithmic challenges. As we will discuss in this article, many commonly used reinforcement learning methods can learn from off-policy data, but such methods often cannot learn effectively from entire offline data, without any additional on-policy interaction. High-dimensional and expressive function approximation generally exacerbates this issue, since function approximation leaves the algorithms vulnerable to distributional shift, one of the central challenges with offline reinforcement learning. However, the appeal of a fully offline reinforcement learning framework is enormous: in the same way that supervised machine learning methods have enabled data to be turned into generalizable and powerful \emph{pattern recognizers} (e.g., image classifiers, speech recognition engines, etc.), offline reinforcement learning methods equipped with powerful function approximation may enable data to be turned into generalizable and powerful \emph{decision making engines}, effectively allowing anyone with a large enough dataset to turn this dataset into a policy that can optimize a desired utility criterion. From healthcare decision-making support to autonomous driving to robotics, the implications of a reliable and effective offline reinforcement learning method would be immense.

In some application domains, the lack of effective offline reinforcement learning methods has driven research in a number of interesting directions. For example, in robotics and autonomous driving, a rapidly growing research topic is the study of simulation to real-world transfer: training policies with reinforcement learning in simulation and then transferring these policies into the real world~\citep{sadeghi2016cad2rl,tan2018sim,chebotar2019closing}. While this approach is very pragmatic (and often effective), its popularity highlights the deficiency in offline reinforcement learning methods: if it was possible to simply train policies with previously collected data, it would likely be unnecessary in many cases to manually design high-fidelity simulators for simulation-to-real-world transfer. After all, outside of reinforcement learning (e.g., in computer vision, NLP, or speech recognition), transfer from simulation is comparatively much less prevalent, since data-driven learning is so effective.

The goal of this article is to provide the reader with the conceptual tools needed to get started on research in the field of offline reinforcement learning (also called batch reinforcement learning~\citep{ernst2005tree,riedmiller2005neural,lange2012batch}), so as to hopefully begin addressing some of these deficiencies. To this end, we will present the offline reinforcement learning problem formulation, and describe some of the challenges associated with this problem setting, particularly in light of recent research on deep reinforcement learning and the interaction between reinforcement learning and high-dimensional function approximator, such as deep networks. We will cover a variety of offline reinforcement learning methods studied in the literature. For each one, we will discuss the conceptual challenges, and initial steps taken to mitigate these challenges. We will then discuss some of the applications of offline reinforcement learning techniques that have already been explored, despite the limitations of current methods, and conclude with some perspectives on future work and open problems in the field.

\section{Offline Reinforcement Learning Problem Statement and Overview}

In this section, we will introduce the mathematical formalism of reinforcement learning and define our notation, and then set up the offline reinforcement learning problem setting, where the goal is to learn near-optimal policies from previously collected data. Then, we will briefly discuss some of the intuition behind why the offline reinforcement learning problem setting poses some unique challenges, using a supervised behavioral cloning example.

\subsection{Reinforcement Learning Preliminaries}
\label{sec:rl_prelims}

In this section, we will define basic reinforcement learning concepts, following standard textbook definitions~\citep{sb-irl-98}. Reinforcement learning addresses the problem of learning to control a dynamical system, in a general sense. The dynamical system is fully defined by a fully-observed or partially-observed Markov decision process (MDP).

\begin{definition}[Markov decision process]
The Markov decision process is defined as a tuple \mbox{$\mdp = (\states,\actions,\transitions,\initstate,\reward,\discount)$}, where $\states$ is a set of states $\bs \in \states$, which may be either discrete or continuous (i.e., multi-dimensional vectors), $\actions$ is a set of actions $\ba \in \actions$, which similarly can be discrete or continuous, $\transitions$ defines a conditional probability distribution of the form $\transitions(\bs_{t+1} | \bs_t, \ba_t)$ that describes the dynamics of the system,\footnote{We will sometimes use time subscripts (i.e., $\bs_{t+1}$ follows $\bs_t$), and sometimes ``prime'' notation (i.e., $\bs'$ is the state that follows $\bs$). Explicit time subscripts can help clarify the notation in finite-horizon settings, while ``prime'' notation is simpler in infinite-horizon settings where absolute time step indices are less meaningful.} $\initstate$ defines the initial state distribution $\initstate(\bs_0)$, $\reward : \states \times \actions \rightarrow \real$ defines a reward function, and $\gamma \in (0, 1]$ is a scalar discount factor.
\end{definition}

We will use the fully-observed formalism in most of this article, though the definition for the partially observed Markov decision process (POMDP) is also provided for completeness. The MDP definition can be extended to the partially observed setting as follows:

\begin{definition}[Partially observed Markov decision process]
The partially observed Markov decision process is defined as a tuple \mbox{$\mdp = (\states,\actions,\observations,\transitions,\initstate,\obsfunc,\reward,\discount)$}, where $\states$, $\actions$, $\transitions$, $\initstate$, $\reward$, and $\discount$ are defined as before, $\observations$ is a set of observations, where each observation is given by $\bo \in \observations$, and $\obsfunc$ is an emission function, which defines the distribution $\obsfunc(\bo_t | \bs_t)$.
\end{definition}

The final goal in a reinforcement learning problem is to learn a policy, which defines a distribution over actions conditioned on states, $\policy(\ba_t | \bs_t)$, or conditioned on observations in the partially observed setting, $\policy(\ba_t | \bo_t)$. The policy may also be conditioned on an observation history, $\policy(\ba_t | \bo_{0:t})$. From these definitions, we can derive the \emph{trajectory distribution}. The trajectory is a sequence of states and actions of length $H$, given by $\tau = (\bs_0, \ba_0, \dots, \bs_H, \ba_H)$, where $H$ may be infinite. The trajectory distribution $p_\policy$ for a given MDP $\mdp$ and policy $\policy$ is given by
\[
p_\policy(\traj) = \initstate(\bs_0) \prod_{t=0}^H \policy(\ba_t | \bs_t) \transitions(\bs_{t+1} | \bs_t, \ba_t).
\]
This definition can easily be extended into the partially observed setting by including the observations $\bo_t$ and emission function $\obsfunc(\bo_t | \bs_t)$. The reinforcement learning objective, $J(\policy)$, can then be written as an expectation under this trajectory distribution:
\begin{equation}
J(\policy) = \E_{\traj \sim p_\policy(\traj)}\left[
\sum_{t=0}^H \discount^t \reward(\bs_t, \ba_t)
\right]. \label{eq:rl_objective}
\end{equation}
When $H$ is infinite, it is sometimes also convenient to assume that the Markov chain on $(\bs_t, \ba_t)$ defined by $\policy(\ba_t | \bs_t) \transitions(\bs_{t+1} | \bs_t, \ba_t)$ is ergodic, and define the objective in terms of the expected reward under the stationary distribution of this Markov chain~\citep{sb-irl-98}. This definition is somewhat complicated by the role of the discount factor. For a full discussion of this topic, we refer the reader to prior work~\citep{thomas}.

In many cases, we will find it convenient to refer to the marginals of the trajectory distribution $p_\policy(\traj)$. We will use $\freq^\policy(\bs)$ to refer to the overall state visitation frequency, averaged over the time steps, and $\freq^\policy_t(\bs_t)$ to refer to the state visitation frequency at time step $t$.

In this section, we will briefly summarize different types of reinforcement learning algorithms and present definitions. At a high level, all standard reinforcement learning algorithms follow the same basic learning loop: the agent \emph{interacts} with the MDP $\mdp$ by using some sort of \emph{behavior policy}, which may or may not match $\policy(\ba|\bs)$, by observing the current state $\bs_t$, selecting an action $\ba_t$, and then observing the resulting next state $\bs_{t+1}$ and reward value $\reward_t = \reward(\bs_t,\ba_t)$. This may repeat for multiple steps, and the agent then uses the observed transitions $(\bs_t,\ba_t,\bs_{t+1},\reward_t)$ to update its policy. This update might also utilize previously observed transitions. We will use $\data = \{ (\bs^i_t,\ba^i_t,\bs^i_{t+1},\reward^i_t) \}$ to denote the set of transitions that are available for the agent to use for updating the policy (``learning''), which may consist of either all transitions seen so far, or some subset thereof.

\paragraph{Policy gradients.} \label{par:policy_gradients} One of the most direct ways to optimize the RL objective in Equation~\ref{eq:rl_objective} is to directly estimate its gradient. In this case, we typically assume that the policy is parameterized by a parameter vector $\theta$, and therefore given by $\policy_\theta(\ba_t|\bs_t)$. For example, $\theta$ might denote the weights of a deep network that outputs the logits for the (discrete) actions $\ba_t$. In this case, we can express the gradient of the objective with respect to $\theta$ as:
\begin{equation}
\nabla_\theta J(\policy_\theta) \!=\! \E_{\traj \sim p_{\policy_\theta}(\traj)}\left[\!
\sum_{t = 0}^H \right. \gamma^t \nabla_\theta \log \policy_\theta(\ba_t|\bs_t) \underbrace{\left(
\sum_{t' = t}^H \gamma^{t'-t} \reward(\bs_{t'},\ba_{t'}) - b(\bs_t)
\right)}_{\text{return estimate } \hat{A}(\bs_t,\ba_t)}
\left. \vphantom{\sum_{t = 0}^H} \!\right],
\label{eq:policy_grad}
\end{equation}
where the return estimator $\hat{A}(\bs_t,\ba_t)$ can itself be learned as a separate neural network \emph{critic}, as discussed below, or it can simply be estimated with Monte Carlo samples, in which case we simply generate samples from $p_{\policy_\theta}(\traj)$, and then sum up the rewards over the time steps of the sampled trajectory. The baseline $b(\bs_t)$ can be estimated as the average reward over the sampled trajectories, or by using a value function estimator $V(\bs_t)$, which we discuss in the dynamic programming section. 
We can equivalently write this gradient expression as an expectation with respect to $\freq^\policy_t(\bs_t)$ as
\begin{equation*}
\nabla_\theta J(\policy_\theta) \!=\! \sum_{t=0}^H \E_{\bs_t \sim \freq^\policy_t(\bs_t), \ba_t \sim \policy_\theta(\ba_t|\bs_t)}\left[
\discount^t \nabla_\theta \log \policy_\theta(\ba_t|\bs_t) \hat{A}(\bs_t,\ba_t)
\right].
\label{eq:policy_grad_state_freq}
\end{equation*}
A common modification is to drop the $\discount^t$ term in front of the gradient, which approximates an average reward setting~\citep{thomas}. Dropping this term and adopting an infinite-horizon formulation, we can further rewrite the policy gradient as expectation under $\freq^\policy(\bs)$ as
\begin{equation*}
\nabla_\theta J(\policy_\theta) \!=\! \frac{1}{1 - \discount} \E_{\bs \sim \freq^\policy(\bs_t), \ba \sim \policy_\theta(\ba|\bs)}\left[
\nabla_\theta \log \policy_\theta(\ba|\bs) \hat{A}(\bs,\ba)
\right].
\label{eq:policy_grad_state_freq_stationary}
\end{equation*}
The constant scaling term $\frac{1}{1 - \discount}$ is often disregarded. This infinite-horizon formulation is often convenient to work with for analyzing and deriving policy gradient methods. For a full derivation of this gradient, we refer the reader to prior work~\citep{sutton2000policy,kakade2002natural,schulman2015trust}. We can summarize a basic Monte Carlo policy gradient algorithm as follows:

\begin{algorithm}[ht]
\caption{On-policy policy gradient with Monte Carlo estimator \label{alg:pg}}
\begin{algorithmic}[1]
\State initialize $\theta_0$
\For{iteration $k \in [0, \dots, K]$}
\State sample trajectories $\{\traj_i\}$ by running $\policy_{\theta_k}(\ba_t|\bs_t)$ \Comment{each $\traj_i$ consists of $\bs_{i,0},\ba_{i,0},\dots,\bs_{i,H},\ba_{i,H}$}
\State compute $\return_{i,t} = \sum_{t'=t}^H \discount^{t'-t} \reward(\bs_{i,t},\ba_{i,t})$
\State fit $b(\bs_t)$ to $\{\return{i,t}\}$ \Comment{use constant $b_t = \frac{1}{N}\sum_i \return{i,t}$, or fit $b(\bs_t)$ to $\{\return{i,t}\}$}
\State compute $\hat{A}(\bs_{i,t},\ba_{i,t}) = \return_{i,t} - b(\bs_t)$
\State estimate $\nabla_{\theta_k} J(\policy_{\theta_k}) \approx \sum_{i,t} \nabla_{\theta_k} \log \policy_{\theta_k}(\ba_{i,t} | \bs_{i,t}) \hat{A}(\bs_{i,t},\ba_{i,t})$
\State update parameters: $\theta_{k+1} \leftarrow \theta_k + \alpha \nabla_{\theta_k} J(\policy_{\theta_k})$
\EndFor
\end{algorithmic}
\end{algorithm}

For additional details on standard on-policy policy gradient methods, we refer the reader to prior work~\citep{sutton2000policy,kakade2002natural,schulman2015trust}.

\paragraph{Approximate dynamic programming.} Another way to optimize the reinforcement learning objective is to observe that, if we can accurately estimate a state or state-action \emph{value function}, it is easy to then recover a near-optimal policy. A value function provides an estimate of the expected cumulative reward that will be obtained by following some policy $\policy(\ba_t|\bs_t)$ when starting from a given state $\bs_t$, in the case of the state-value function $V^\policy(\bs_t)$, or when starting from a state-action tuple $(\bs_t,\ba_t)$, in the case of the state-action value function $Q^\policy(\bs_t,\ba_t)$. We can define these value functions as:
\begin{align*}
V^\policy(\bs_t) &= \E_{\traj \sim p_\policy(\traj | \bs_t)} \left[
\sum_{t' = t}^H \gamma^{t' - t} \reward(\bs_t, \ba_t)
\right] \\
Q^\policy(\bs_t,\ba_t) &= \E_{\traj \sim p_\policy(\traj | \bs_t, \ba_t)} \left[
\sum_{t' = t}^H \gamma^{t' - t} \reward(\bs_t, \ba_t)
\right].
\end{align*}
From this, we can derive recursive definitions for these value functions, which are given as
\begin{align*}
V^\policy(\bs_t) &= \E_{\ba_t \sim \policy(\ba_t | \bs_t)}\left[
Q^\policy(\bs_t,\ba_t)
\right] \\
Q^\policy(\bs_t, \ba_t) &= \reward(\bs_t,\ba_t) + \discount \E_{\bs_{t+1} \sim \transitions(\bs_{t+1}|\bs_t,\ba_t)}\left[
V^\policy(\bs_{t+1})
\right].
\end{align*}
We can combine these two equations to express the $Q^\policy(\bs_t,\ba_t)$ in terms of $Q^\policy(\bs_{t+1},\ba_{t+1})$:
\begin{align}
Q^\policy(\bs_t, \ba_t) &= \reward(\bs_t,\ba_t) + \discount \E_{\bs_{t+1} \sim \transitions(\bs_{t+1}|\bs_t,\ba_t), \ba_{t+1} \sim \policy(\ba_{t+1} | \bs_{t+1})}\left[
Q^\policy(\bs_{t+1},\ba_{t+1}))
\right]. \label{eq:qeq}
\end{align}
We can also express these in terms of the \emph{Bellman operator} for the policy $\policy$, which we denote $\bellman^\policy$. For example, Equation~(\ref{eq:qeq}) can be written as $\vec{Q}^\policy = \bellman^\policy \vec{Q}^\policy$, where $\vec{Q}^\policy$ denotes the Q-function $Q^\policy$ represented as a vector of length $|\states|\times |\actions|$. Before moving on to deriving learning algorithms based on these definitions, we briefly discuss some properties of the Bellman operator. This Bellman operator has a unique fixed point that corresponds to the true Q-function for the policy $\policy(\ba|\bs)$, which can be obtained by repeating the iteration $\vec{Q}^\policy_{k+1} = \bellman^\policy \vec{Q}^\policy_k$, and it can be shown that $\lim_{k \rightarrow \infty} \vec{Q}^\policy_k = \vec{Q}^\policy$, which obeys Equation~(\ref{eq:qeq})~\citep{sb-irl-98}. The proof for this follows from the observation that $\bellman^\policy$ is a contraction in the $\ell_\infty$ norm~\citep{lagoudakis2003least}.

Based on these definitions, we can derive two commonly used algorithms based on dynamic programming: Q-learning and actor-critic methods. To derive Q-learning, we express the policy implicitly in terms of the Q-function, as \mbox{$\policy(\ba_t | \bs_t) = \delta(\ba_t = \arg\max Q(\bs_t,\ba_t))$}, and only learn the Q-function $Q(\bs_t,\ba_t)$. By substituting this (implicit) policy into the above dynamic programming equation, we obtain the following condition on the optimal Q-function:
\begin{equation}
Q^\star(\bs_t, \ba_t) = \reward(\bs_t,\ba_t) + \discount \E_{\bs_{t+1} \sim \transitions(\bs_{t+1}|\bs_t,\ba_t)}\left[
\max_{\ba_{t+1}} Q^\star(\bs_{t+1},\ba_{t+1})
\right]. \label{eq:q_learning_equation}
\end{equation}
We can again express this as $\vec{Q} = \bellman^\star \vec{Q}$ in vector notation, where $\bellman^\star$ now refers to the Bellman optimality operator. Note however that this operator is not linear, due to the maximization on the right-hand side in Equation~(\ref{eq:q_learning_equation}). To turn this equation into a learning algorithm, we can minimize the difference between the left-hand side and right-hand side of this equation with respect to the parameters of a parametric Q-function estimator with parameters $\phi$, $Q_\phi(\bs_t,\ba_t)$. There are a number of variants of this Q-learning procedure, including variants that fully minimize the difference between the left-hand side and right-hand side of the above equation at each iteration, commonly referred to as fitted Q-iteration~\citep{ernst2005tree,riedmiller2005neural}, and variants that take a single gradient step, such as the original Q-learning method~\citep{watkins1992q}. The commonly used variant in deep reinforcement learning is a kind of hybrid of these two methods, employing a replay buffer~\citep{lin1992self} and taking gradient steps on the Bellman error objective concurrently with data collection~\citep{mnih2013playing}. We write out a general recipe for Q-learning methods in Algorithm~\ref{alg:qlearning}.

\begin{algorithm}[ht]
\caption{Generic Q-learning (includes FQI and DQN as special cases) \label{alg:qlearning}}
\begin{algorithmic}[1]
\State initialize $\phi_0$
\State initialize $\policy_0(\ba|\bs) = \epsilon \mathcal{U}(\ba) + (1-\epsilon)\delta(\ba = \arg\max_{\ba} Q_{\phi_0}(\bs,\ba))$ \Comment{Use $\epsilon$-greedy exploration}
\State initialize replay buffer $\data = \emptyset$ as a ring buffer of fixed size
\State initialize $\bs \sim \initstate(\bs)$
\For{iteration $k \in [0, \dots, K]$}
\For{step $s \in [0, \dots, S-1]$}
\State $\ba \sim \policy_k(\ba|\bs)$ \Comment{sample action from exploration policy}
\State $\bs' \sim p(\bs' | \bs, \ba)$ \Comment{sample next state from MDP}
\State $\data \leftarrow \data \cup \{(\bs,\ba,\bs',\reward(\bs,\ba))\}$ \Comment{append to buffer, purging old data if buffer too big}
\EndFor
\State $\phi_{k,0} \leftarrow \phi_k$
\For{gradient step $g \in [0, \dots, G-1]$}
\State sample batch $\batch \subset \data$ \Comment{$B = \{ (\bs_i, \ba_i, \bs'_i, r_t ) \}$}
\State estimate error $\en(B,\phi_{k,g}) = \sum_i \left( Q_{\phi_{k,g}} - (r_i + \discount \max_{\ba'} Q_{\phi_k}(\bs',\ba')) \right)^2$
\State update parameters: $\phi_{k,g+1} \leftarrow \phi_{k,g} - \alpha \nabla_{\phi_{k,g}} \en(B,\phi_{k,g})$
\EndFor
\State $\phi_{k+1} \leftarrow \phi_{k,G}$ \Comment{update parameters}
\EndFor
\end{algorithmic}
\end{algorithm}

Classic Q-learning can be derived as the limiting case where the buffer size is 1, and we take $G=1$ gradient steps and collect $S=1$ transition samples per iteration, while classic fitted Q-iteration runs the inner gradient descent phase to convergence (i.e., $G=\infty$), and uses a buffer size equal to the number of sampling steps $S$. Note that many modern implementations also employ a \emph{target network}, where the target value $r_i + \discount \max_{\ba'} Q_{\phi_k}(\bs',\ba')$ actually uses $\phi_{L}$, where $L$ is a lagged iteration (e.g., the last $k$ that is a multiple of 1000). Note that these approximations violate the assumptions under which Q-learning algorithms can be proven to converge. However, recent work suggests that high-capacity function approximators, which correspond to a very large set $\qset$, generally do tend to make this method convergent in practice, yielding a Q-function that is close to $\vec{Q}^\policy$~\citep{fu2019diagnosing,van2018deep}.

\paragraph{Actor-critic algorithms.} Actor-critic algorithms combine the basic ideas from policy gradients and approximate dynamic programming. Such algorithms employ \emph{both} a parameterized policy and a parameterized value function, and use the value function to provide a better estimate of $\hat{A}(\bs,\ba)$ for policy gradient calculation. There are a number of different variants of actor-critic methods, including on-policy variants that directly estimate $V^\policy(\bs)$~\citep{konda2000actor}, and off-policy variants that estimate $Q^\policy(\bs,\ba)$ via a parameterized state-action value function $Q^\policy_\phi(\bs,\ba)$~\citep{sac,haarnoja2017reinforcement,heess2015learning}. We will focus on the latter class of algorithms, since they can be extended to the offline setting. The basic design of such an algorithm is a straightforward combination of the ideas in dynamic programming and policy gradients. Unlike Q-learning, which directly attempts to learn the optimal Q-function, actor-critic methods aim to learn the Q-function corresponding to the current parameterized policy $\policy_\theta(\ba | \bs)$, which must obey the equation
\[
Q^\policy(\bs_t, \ba_t) = \reward(\bs_t,\ba_t) + \discount \E_{\bs_{t+1} \sim \transitions(\bs_{t+1}|\bs_t,\ba_t), \ba_{t+1} \sim \policy_\theta(\ba_{t+1} | \bs_{t+1})}\left[
Q^\policy(\bs_{t+1},\ba_{t+1})
\right].
\]
As before, this equation can be expressed in vector form in terms of the Bellman operator for the policy, $\vec{Q}^\policy = \bellman^\policy \vec{Q}^\policy$, where $\vec{Q}^\policy$ denotes the Q-function $Q^\policy$ represented as a vector of length $|\states|\times |\actions|$. We can now instantiate a complete algorithm based on this idea, shown in Algorithm~\ref{alg:actorcritic}.

\begin{algorithm}[ht]
\caption{Generic off-policy actor-critic \label{alg:actorcritic}}
\begin{algorithmic}[1]
\State initialize $\phi_0$
\State initialize $\theta_0$
\State initialize replay buffer $\data = \emptyset$ as a ring buffer of fixed size
\State initialize $\bs \sim \initstate(\bs)$
\For{iteration $k \in [0, \dots, K]$}
\For{step $s \in [0, \dots, S-1]$}
\State $\ba \sim \policy_{\theta_k}(\ba|\bs)$ \Comment{sample action from current policy}
\State $\bs' \sim p(\bs' | \bs, \ba)$ \Comment{sample next state from MDP}
\State $\data \leftarrow \data \cup \{(\bs,\ba,\bs',\reward(\bs,\ba))\}$ \Comment{append to buffer, purging old data if buffer too big}
\EndFor
\State $\phi_{k,0} \leftarrow \phi_k$
\For{gradient step $g \in [0, \dots, G_Q-1]$}
\State sample batch $\batch \subset \data$ \Comment{$B = \{ (\bs_i, \ba_i, \bs'_i, r_t ) \}$}
\State estimate error $\en(B,\phi_{k,g}) = \sum_i \left( Q_{\phi_{k,g}} - (r_i + \discount \E_{\ba'\sim \policy_k(\ba'|\bs')} Q_{\phi_k}(\bs',\ba')) \right)^2$
\State update parameters: $\phi_{k,g+1} \leftarrow \phi_{k,g} - \alpha_Q \nabla_{\phi_{k,g}} \en(B,\phi_{k,g})$
\EndFor
\State $\phi_{k+1} \leftarrow \phi_{k,G_Q}$ \Comment{update Q-function parameters}
\State $\theta_{k,0} \leftarrow \theta_k$
\For{gradient step $g \in [0,\dots G_\policy-1]$}
\State sample batch of states $\{\bs_i\}$ from $\data$
\State for each $\bs_i$, sample $\ba_i \sim \policy_{\theta_{k,g}}(\ba|\bs_i)$ \Comment{do not use actions in the buffer!}
\State for each $(\bs_i,\ba_i)$, compute $\hat{A}(\bs_i, \ba_i) = Q_{\phi_{k+1}}(\bs_i,\ba_i) - \E_{\ba\sim\policy_{k,g}(\ba|\bs_i)}[Q_{\phi_{k+1}}(\bs_i,\ba)]$
\State $\nabla_{\theta_{k,g}} J(\policy_{\theta_{k,g}}) \approx \frac{1}{N} \nabla_{\theta_{k,g}} \log \policy_{\theta_{k,g}}(\bs_i, \ba_i) \hat{A}(\bs_i, \ba_i)$
\State $\theta_{k,g+1} \leftarrow \theta_{k,g} + \alpha_\policy \nabla_{\theta_{k,g}} J(\policy_{\theta_{k,g}})$
\EndFor
\State $\theta{k+1} \leftarrow \theta_{k,G_\policy}$ \Comment{update policy parameters}
\EndFor
\end{algorithmic}
\end{algorithm}

For more details, we refer the reader to standard textbooks and prior works~\citep{sb-irl-98,konda2000actor}.
Actor-critic algorithms are closely related with another class of methods that frequently arises in dynamic programming, called policy iteration (PI)~\citep{lagoudakis2003least}. Policy iteration consists of two phases: policy evaluation and policy improvement. The policy evaluation phase computes the Q-function for the current policy $\policy$, $Q^\policy$, by solving for the fixed point such that $Q^\policy = \bellman^\policy Q^\policy$. This can be done via linear programming or solving a system of linear equations, as we will discuss in Section~\ref{sec:adp}, or via gradient updates, analogously to line 15 in Algorithm~\ref{alg:actorcritic}. The next policy iterate is then computed in the policy improvement phase, by choosing the action that greedily maximizes the Q-value at each state, such that $\policy_{k+1}(\ba|\bs) = \delta(\ba = \arg\max_{\ba} Q^{\policy_k}(\bs,\ba))$, or by using a gradient based update procedure as is employed in Algorithm~\ref{alg:actorcritic} on line 24. In the absence of function approximation (e.g., with tabular representations) policy iteration produces a monotonically improving sequence of policies, and converges to the optimal policy. Policy iteration can be obtained as a special case of the generic actor-critic algorithm in Algorithm~\ref{alg:actorcritic} when we set $G_Q = \infty$ and $G_\pi = \infty$, when the buffer $\mathcal{D}$ consists of each and every transition of the MDP.


\paragraph{Model-based reinforcement learning.} Model-based reinforcement learning is a general term that refers to a broad class of methods that utilize explicit estimates of the transition or dynamics function $\transitions(\bs_{t+1}|\bs_t,\ba_t)$, parameterized by a parameter vector $\psi$, which we will denote $\transitions_\psi(\bs_{t+1}|\bs_t,\ba_t)$. There is no single recipe for a model-based reinforcement learning method. Some commonly used model-based reinforcement learning algorithms learn only the dynamics model $\transitions_\psi(\bs_{t+1}|\bs_t,\ba_t)$, and then utilize it for planning at test time, often by means of model-predictive control (MPC)~\citep{tassa2012synthesis} with various trajectory optimization methods~\citep{nagabandi2018neural,pets}. Other model-based reinforcement learning methods utilize a learned policy $\policy_\theta(\ba_t|\bs_t)$ in addition to the dynamics model, and employ backpropagation through time to optimize the policy with respect to the expected reward objective~\citep{deisenroth2011pilco}. Yet another set of algorithms employ the model to generate ``synthetic'' samples to augment the sample set available to model-free reinforcement learning methods. The classic Dyna algorithm uses this recipe in combination with Q-learning and one-step predictions via the model from previously seen states~\citep{sutton1991dyna}, while a variety of recently proposed algorithms employ synthetic model-based rollouts with policy gradients~\citep{pipps,simpl} and actor-critic algorithms~\citep{mbpo}. Since there are so many variants of model-based reinforcement learning algorithms, we will not go into detail on each of them in this section, but we will discuss some considerations for offline model-based reinforcement learning in Section~\ref{sec:model_based}.

\subsection{Offline Reinforcement Learning}

The offline reinforcement learning problem can be defined as a \emph{data-driven} formulation of the reinforcement learning problem. The end goal is still to optimize the objective in Equation~(\ref{eq:rl_objective}). However, the agent no longer has the ability to interact with the environment and collect additional transitions using the behavior policy. Instead, the learning algorithm is provided with a \emph{static} dataset of transitions, $\data = \{ (\bs^i_t,\ba^i_t,\bs^i_{t+1},\reward^i_t) \}$, and must learn the best policy it can using this dataset. This formulation more closely resembles the standard supervised learning problem statement, and we can regard $\data$ as the \emph{training set} for the policy. In essence, offline reinforcement learning requires the learning algorithm to derive a sufficient understanding of the dynamical system underlying the MDP $\mdp$ entirely from a fixed dataset, and then construct a policy $\policy(\ba|\bs)$ that attains the largest possible cumulative reward \emph{when it is actually used to interact with the MDP}. We will use $\behavior$ to denote the distribution over states and actions in $\data$, such that we assume that the state-action tuples $(\bs,\ba) \in \data$ are sampled according to $\bs \sim \freq^{\behavior}(\bs)$, and the actions are sampled according to the behavior policy, such that $\ba \sim \behavior(\ba|\bs)$.

This problem statement has been presented under a number of different names. The term ``off-policy reinforcement learning'' is typically used as an umbrella term to denote all reinforcement learning algorithms that can employ datasets of transitions $\data$ where the corresponding actions in each transition were collected with any policy \emph{other} than the current policy $\policy(\ba|\bs)$. Q-learning algorithms, actor-critic algorithms that utilize Q-functions, and many model-based reinforcement learning algorithm are off-policy algorithms. However, off-policy algorithms still often employ additional interaction (i.e., online data collection) during the learning process. Therefore, the term ``fully off-policy'' is sometimes used to indicate that no additional online data collection is performed. Another commonly used term is ``batch reinforcement learning''~\citep{ernst2005tree,riedmiller2005neural,lange2012batch}. While this term has been used widely in the literature, it can also cause some amount of confusion, since the use of a ``batch'' in an iterative learning algorithm can also refer to a method that consumes a batch of data, updates a model, and then obtains a different batch, as opposed to a traditional online learning algorithm, which consumes one sample at a time. In fact, \citet{lange2012batch} further introduces qualifiers  ``pure'' and ``growing'' batch reinforcement learning to clarify this. 
To avoid this confusion, we will instead use the term ``offline reinforcement learning'' in this tutorial.

The offline reinforcement learning problem can be approached using algorithms from each of the four categories covered in the previous section, and in principle any off-policy RL algorithm \emph{could} be used as an offline RL algorithm. For example, a simple offline RL method can be obtained simply by using Q-learning without additional online exploration, using $\data$ to pre-populate the data buffer. This corresponds to changing the initialization of $\data$ in Algorithm~\ref{alg:qlearning}, and setting $S=0$. However, as we will discuss later, not all such methods are effective in the offline setting.

\subsection{Example Scenarios}

Before delving deeper into the technical questions surrounding offline RL, we will first discuss a few example scenarios where offline RL might be utilized. These scenarios will help us to understand the factors that we must consider when designing offline RL methods that are not only convergent and principled, but also likely to work well in practice. A more complete discussion of actual applications is provided in Section~\ref{sec:applications}.

\paragraph{Decision making in health care.} An example health care scenario might formulate a Markov decision process to model the process of diagnosing and treating a patient, where actions correspond to various available interventions (e.g., diagnostic tests and treatments), and observations correspond to the patient's symptoms and results of diagnostic tests. A partially observed MDP formulation may be most suitable in such cases. Conventional active reinforcement learning in such scenarios might be prohibitively dangerous -- even utilizing a fully trained policy to treat a patient is a difficult prospect for clinicians, and deploying a partially trained policy would be out of the question. Therefore, offline RL might be the only viable path to apply reinforcement learning in such settings. Offline data would then be obtained from treatment histories of real patients, with the ``actions'' that were selected by their physician.

\paragraph{Learning goal-directed dialogue policies.} Dialogue can be viewed as interactive sequential decision making problem, which can also be modeled as an MDP, particularly when the dialogue is goal-directed (e.g., a chat bot on an e-commerce website that is offering information about a product to persuade a user to make a purchase). However, since the goal for such agents is to interact successfully with real humans, collecting trials requires interacting with live humans, which may be prohibitively expensive at the scale needed to train effective conversational agents. However, offline data can be collected directly from humans, and may indeed be natural to collect in any application domain where the aim is to partially or completely supplant human operators, by recording the interactions that are already taking place with the human-operated system.

\paragraph{Learning robotic manipulation skills.} In a robotic manipulation setting, active reinforcement learning may in fact be quite feasible. However, we might want to learn policies for a variety of robotic skills (e.g., all of the steps necessary to prepare a variety of meals for a home cooking robot) that generalize effectively over different environments and settings. In that case, each skill by itself might require a very large amount of interaction, as we would not only need to collect enough data to learn the skill, but enough data such that this skill generalizes effectively to all the situations (e.g., all the different homes) in which the robot might need to perform it. With offline RL, we could instead imagine including all of the data the robot has ever collected for \emph{all} of its previously learned skills in the data buffer for each new skill that it learns. In this way, some skills could conceivably be acquired without any new data collection, if they can be assembled from parts of previously learned behaviors (e.g., cooking a soup that includes onions and carrots can likely be learned from experience cooking a soup with onions and meat, as well as another soup with carrots and cucumbers). In this way, offline RL can effectively utilize \emph{multi-task} data.

\subsection{What Makes Offline Reinforcement Learning Difficult?}
\label{sec:basic_challenges}

Offline reinforcement learning is a difficult problem for multiple reasons, some of which are reasonably clear, and some of which might be a bit less clear. Arguably the most obvious challenge with offline reinforcement learning is that, because the learning algorithm must rely entirely on the static dataset $\data$, there is no possibility of improving exploration: exploration is \emph{outside} the scope of the algorithm, so if $\data$ does not contain transitions that illustrate high-reward regions of the state space, it may be impossible to discover those high-reward regions. However, because there is nothing that we can do to address this challenge, we will not spend any more time on it, and will instead assume that $\data$ adequately covers the space of high-reward transitions to make learning feasible.\footnote{It is worth noting that defining this notion formally is itself an open problem, and to the best knowledge of the authors, there are no known non-trivial ``sufficiency'' conditions on $\data$ that allows us to formulate guarantees that any offline reinforcement learning algorithm will recover an optimal or near-optimal policy.}

A more subtle but practically more important challenge with offline reinforcement learning is that, at its core, offline reinforcement learning is about making and answering counterfactual queries. Counterfactual queries are, intuitively, ``what if'' questions. Such queries require forming hypotheses about what \emph{might} happen if the agent were to carry out a course of action different from the one seen in the data. This is a necessity in offline RL, since we if we want the learned policy to perform better than the behavior seen in the dataset $\data$, we must execute a sequence of actions that is in some way different. Unfortunately, this strains the capabilities of many of our current machine learning tools, which are designed around the assumption that the data is independent and identically distributed (i.i.d.). That is, in standard supervised learning, the goal is to train a model that attains good performance (e.g., high accuracy) on data coming from the same distribution as the training data. In offline RL, the whole point is to learn a policy that does something \emph{differently} (presumably better) from the pattern of behavior observed in the dataset $\data$.

The fundamental challenge with making such counterfactual queries is distributional shift: while our function approximator (policy, value function, or model) might be trained under one distribution, it will be evaluated on a different distribution, due both to the change in visited states for the new policy and, more subtly, by the act of maximizing the expected return. This latter point is discussed in more detail in Section~\ref{sec:adp}. Distributional shift issues can be addressed in several ways, with the simplest one being to constrain something about the learning process such that the distributional shift is bounded. For example, by constraining how much the learned policy $\policy(\ba|\bs)$ differs from the behavior policy $\behavior(\ba|\bs)$, we can bound state distributional shift~\citep{kakade2002approximately,schulman2015trust}.

In this section, we will provide a short theoretical illustration of how harmful distributional shift can be on the performance of policies in MDPs. In this example, based on \cite{ross2011reduction}, we will assume that we are provided with optimal action labels $\ba^\star$ at each state $\bs \in \data$. One might expect that, under such a strong assumption, the performance of our learned policy should be at least as good as the policies that we can learn with reinforcement learning without such optimal action labels. The goal in this analysis will be to bound the number of mistakes made by the learned policy $\policy(\ba|\bs)$ based on this labeled dataset, denoted as
\[
\ell(\policy) = \E_{p_\policy(\traj)}\left[ \sum_{t=0}^H \delta(\ba_t \neq \ba_t^\star) \right].
\]
If we train $\policy(\ba|\bs)$ with supervised learning (i.e., standard empirical risk minimization) on this labeled dataset, we have the following result from \cite{ross2011reduction}:
\begin{theorem}[Behavioral cloning error bound]\label{thm:cloning}
If $\policy(\ba|\bs)$ is trained via empirical risk minimization on $\bs \sim \freq^{\behavior}(\bs)$ and optimal labels $\ba^\star$, and attains generalization error $\epsilon$ on $\bs \sim \freq^{\behavior}(\bs)$, then $\ell(\policy) \leq C + H^2\epsilon$ is the best possible bound on the expected error of the learned policy.
\end{theorem}
\begin{proof}
The proof follows from Theorem 2.1 from \cite{ross2011reduction} using the 0-1 loss, and the bound is the best possible bound following the example from \cite{ross2010efficient}.
\end{proof}
Interestingly, if we allow for additional data collection, where we follow the learned policy $\policy(\ba|\bs)$ to gather additional states $\bs \sim \freq^\policy(\bs)$, and then access optimal action labels for these new \emph{on-policy} states, the best possible bound becomes substantially better:
\begin{theorem}[DAgger error bound]
If $\policy(\ba|\bs)$ is trained via empirical risk minimization on $\bs \sim \freq^\policy(\bs)$ and optimal labels $\ba^\star$, and attains generalization error $\epsilon$ on $\bs \sim \freq^\policy(\bs)$, then $\ell(\policy) \leq C + H\epsilon$ is the best possible bound on the expected error of the learned policy.
\end{theorem}
\begin{proof}
The proof follows from Theorem 3.2 from \cite{ross2011reduction}. This is the best possible bound, because the probability of a mistake at any time step is at least $\epsilon$, and $\sum_{t=1}^H \epsilon = H\epsilon$.
\end{proof}
This means that, even with optimal action labels, we get an error bound that is at best quadratic in the time horizon $H$ in the offline case, but linear in $H$ in the online case. Intuitively, the reason for this gap in performance is that, in the offline case, the learned policy $\policy(\ba|\bs)$ may enter into states that are far outside of its training distribution, since $\freq^\policy(\bs)$ may be very different from $\freq^{\behavior}(\bs)$. In these out-of-distribution states, the generalization error bound $\epsilon$ no longer holds, since standard empirical risk minimization makes no guarantees about error when encountering out-of-distribution inputs that were not seen during training. Once the policy enters one of these out-of-distribution states, it will keep making mistakes and may remain out-of-distribution for the remainder of the trial, accumulating $O(H)$ error. Since there is a non-trivial chance of entering an out-of-distribution state at every one of the $H$ time steps, the overall error therefore scales as $O(H^2)$. In the on-policy case, such out-of-distribution states are not an issue. Of course, this example is somewhat orthogonal to the main purpose of this tutorial, which is to study offline reinforcement learning algorithms, rather than offline behavioral cloning methods. However, it should serve as a warning, as it indicates that the challenges of distributional shift are likely to cause considerable harm to any policy trained from an offline dataset if care is not taken to minimize its detrimental effects.

\section{Offline Evaluation and Reinforcement Learning via Importance Sampling}

In this section, we survey offline reinforcement learning algorithms based on direct estimation of policy return. These methods generally utilize some form of importance sampling to either evaluate the return of a given policy, or to estimate the corresponding policy gradient, corresponding to an offline variant of the policy gradient methods discussed in Section~\ref{sec:rl_prelims}. The most direct way to utilize this idea is to employ importance sampling to estimate $J(\pi)$ with trajectories sampled from $\behavior(\traj)$. This is known as \emph{off-policy evaluation}. In principle, once we can evaluate $J(\pi)$, we can select the most performant policy. In this section, we review approaches for off-policy evaluation with importance sampling and then discuss how these ideas can be used for offline reinforcement learning.

\subsection{Off-Policy Evaluation via Importance Sampling}
\label{sec:ope}

We can na\"{i}vely use importance sampling to derive an unbiased estimator of $J(\pi)$ that relies on off-policy trajectories:
\begin{align}
J(\policy_\theta) &= \E_{\traj \sim \behavior(\traj)}\left[
\frac{\policy_\theta(\traj)}{\behavior(\traj)}\sum_{t = 0}^H \discount^t \reward(\bs, \ba) \right] \nonumber \\
&= \E_{\traj \sim \behavior(\traj)}\left[
\left( \prod_{t=0}^H \frac{\policy_\theta(\ba_t | \bs_t)}{\behavior(\ba_t | \bs_t)} \right)\sum_{t = 0}^H \discount^t r(\bs, \ba) \right] \approx \sum_{i=1}^n w_H^i \sum_{t = 0}^H \discount^t r_t^i, \label{eq:return_is}
\end{align}
where $w_t^i = \frac{1}{n}\prod_{t'=0}^t \frac{\policy_\theta(\ba_{t'}^i | \bs_{t'}^i)}{\behavior(\ba_{t'}^i | \bs_{t'}^i)}$ and $\{(\bs_0^i, \ba_0^i, r_0^i, \bs_1^i, \ldots)\}_{i=1}^n$ are $n$ trajectory samples from $\behavior(\traj)$ \citep{precup2000eligibility}. Unfortunately, such an estimator can have very high variance (potentially unbounded if $H$ is infinite) due to the product of importance weights. Self-normalizing the importance weights (i.e., dividing the weights by $\sum_{i=1}^n w_H^i$) results in the \emph{weighted importance sampling} estimator~\citep{precup2000eligibility}, which is biased, but can have much lower variance and is still a strongly consistent estimator.

To improve this estimator, we need to take advantage of the statistical structure of the problem. Because $\reward_t$ does not depend on $\bs_{t'}$ and $\ba_{t'}$ for $t' > t$, we can drop the importance weights from future time steps, resulting in the \emph{per-decision} importance sampling estimator~\citep{precup2000eligibility}:
\begin{equation*}
 J(\pi_\theta) = \E_{\traj \sim \behavior(\traj)}\left[
\sum_{t = 0}^H \left( \prod_{t'=0}^{t} \frac{\policy_\theta(\ba_t | \bs_t)}{\behavior(\ba_t | \bs_t)} \right) \gamma^t \reward(\bs, \ba) \right] \approx \frac{1}{n} \sum_{i=1}^n \sum_{t = 0}^H w_t^i \gamma^t \reward_t^i.   
\end{equation*}
As before, this estimator can have high variance, and we can form a weighted per-decision importance estimator by normalizing the weights. Unfortunately, in many practical problems, the weighted per-decision importance estimator still has too much variance to be effective.

If we have an approximate model that can be used to obtain an approximation to the Q-values for each state-action tuple $(\bs_t,\ba_t)$, which we denote $\hat{Q}^\policy(\bs_t, \ba_t)$, we can incorporate it into this estimate. Such an estimate can be obtained, for example, by estimating a model of the MDP transition probability $\transitions(\bs_{t+1}|\bs_t,\ba_t)$ and then solving for the corresponding Q-function, or via any other method for approximating Q-values. We can incorporate these estimates as control variates into the importance sampled estimator to get the best of both:
\begin{equation}
    J(\policy_\theta) \approx \sum_{i=1}^n \sum_{t = 0}^H \discount^t \left( w_t^i \left(r_t^i - \hat{Q}^{\policy_\theta}(\bs_t, \ba_t) \right) - w_{t-1}^i\E_{\ba \sim \pi_\theta(\ba | \bs_t)} \left[ \hat{Q}^{\policy_\theta}(\bs_t, \ba)\right] \right). \label{eq:doubly_robust}
\end{equation}
This is known as the doubly robust estimator~\citep{jiang2015doubly,thomas2016data} because it is unbiased if either $\behavior$ is known or if the model is correct. We can also form a weighted version by normalizing the weights. More sophisticated estimators can be formed by training the model with knowledge of the policy to be evaluated~\citep{farajtabar2018more}, and by trading off bias and variance more optimally~\citep{thomas2016data,wang2017optimal}.

Beyond consistency or unbiased estimates, we frequently desire a (high-probability) guarantee on the performance of a policy. \citet{thomas2015high} derived confidence bounds based on concentration inequalities specialized to deal with the high variance and potentially large range of the importance weighted estimators. Alternatively, we can construct confidence bounds based on distributional assumptions (e.g., asymptotic normality)~\citep{thomas2015high} or via bootstrapping~\citep{thomas2015high,hanna2017bootstrapping} which may be less conservative at the cost of looser guarantees.

Such estimators can also be utilized for policy improvement, by searching over policies with respect to their estimated return. In safety-critical applications of offline RL, we would like to improve over the behavior policy with a guarantee that with high probability our performance is no lower than a bound. \citet{thomas2015high} show that we can search for policies using lower confidence bounds on importance sampling estimators to ensure that the safety constraint is met. Alternatively, we can search over policies in a model of the MDP and bound the error of the estimated model with high probability~\citep{ghavamzadeh2016safe,laroche2017safe,nadjahi2019safe}.

\subsection{The Off-Policy Policy Gradient}
\label{sec:offline_pg}

Importance sampling can also be used to directly estimate the policy gradient, rather than just obtaining an estimate of the value for a given policy. As discussed in Section~\ref{sec:rl_prelims}, policy gradient methods aim to optimize $J(\policy)$ by computing estimates of the gradient with respect to the policy parameters. We can estimate the gradient with Monte Carlo samples, as in Equation (\ref{eq:policy_grad}), but this requires \emph{on-policy} trajectories (i.e., $\traj \sim \policy_\theta(\tau)$). Here, we extend this approach to the offline setting.

Previous work has generally focused on the \emph{off-policy} setting, where trajectories are sampled from a distinct behavior policy $\behavior(\ba | \bs) \neq \policy(\ba | \bs)$. However, in contrast to the offline setting, such methods assume we can continually sample new trajectories from $\behavior$, while old trajectories are reused for efficiency. We begin with the off-policy setting, and then discuss additional challenges in extending such methods to the offline setting.

Noting the similar structure between $J(\pi)$ and the policy gradient, we can adapt the techniques for estimating $J(\pi)$ off-policy to the policy gradient
\begin{align*}
\nabla_\theta J(\policy_\theta) &= \E_{\traj \sim \behavior(\traj)}\left[
\frac{{\policy_\theta}(\traj)}{\behavior(\traj)}\sum_{t = 0}^H \right. \gamma^t \nabla_\theta \log \policy_\theta(\ba_t|\bs_t) \hat{A}(\bs_t, \ba_t)
\left. \vphantom{\sum_{t = 0}^H} \right] \\
&=\E_{\traj \sim \behavior(\traj)}\left[
\left( \prod_{t=0}^H \frac{\policy_\theta(\ba_t | \bs_t)}{\behavior(\ba_t | \bs_t)} \right) \sum_{t = 0}^H \right. \gamma^t \nabla_\theta \log \policy_\theta(\ba_t|\bs_t)  \hat{A}(\bs_t, \ba_t)
\left. \vphantom{\sum_{t = 0}^H} \right] \\
&\approx \sum_{i=1}^n w_H^i \sum_{t = 0}^H \gamma^t \nabla_\theta \log \policy_\theta(\ba_t^i|\bs_t^i)  \hat{A}(\bs_t^i, \ba_t^i),
\end{align*}
where $\{(\bs_0^i, \ba_0^i, r_0^i, \bs_1^i, \ldots)\}_{i=1}^n$ are $n$ trajectory samples from $\behavior(\traj)$ \citep{precup2000eligibility,precup2001off,peshkin2002learning}. Similarly, we can self-normalize the importance weights resulting in the \emph{weighted importance sampling} policy gradient estimator~\citep{precup2000eligibility}, which is biased, but can have much lower variance and is still a consistent estimator.

If we use the Monte Carlo estimator with baseline for $\hat{A}$ (i.e., \mbox{$\hat{A}(\bs_t^i, \ba_t^i) = \sum_{t' = t}^H \gamma^{t' - t}r_{t'} - b(\bs_t^i) $}), then because $r_t$ does not depend on $\bs_{t'}$ and $\ba_{t'}$ for $t' > t$, we can drop importance weights in the future, resulting in the \emph{per-decision} importance sampling policy gradient estimator~\citep{precup2000eligibility}:
\begin{equation*}
\nabla_\theta J(\policy_\theta) \approx \sum_{i=1}^n \sum_{t = 0}^H w_t^i \gamma^t \left( \sum_{t' = t}^H \gamma^{t' - t}  \frac{w_{t'}^i}{w_t^i} r_{t'} - b(\bs_t^i) \right)\nabla_\theta \log \policy_\theta(\ba_t^i|\bs_t^i).   
\end{equation*}
As before, this estimator can have high variance, and we can form a weighted per-decision importance estimator by normalizing the weights. Paralleling the development of doubly robust estimators for policy evaluation, doubly robust estimators for the policy gradient have also been derived~\citep{gu2017deep,huang2019importance,pankov2018reward,cheng2019trajectory}.
Unfortunately, in many practical problems, these estimators have too high variance to be effective. 

Practical off-policy algorithms derived from such estimators can also employ regularization, such that the learned policy $\policy_\theta(\ba|\bs)$ does not deviate too far from the behavior policy $\behavior(\ba|\bs)$, thus keeping the variance of the importance weights from becoming too large. One example of such a regularizer is the \emph{soft max} over the (unnormalized) importance weights~\citep{levine2013guided}. This regularized gradient estimator $\nabla_\theta \bar{J}(\policy_\theta)$ has the following form:
\[
\nabla_\theta \bar{J}(\policy_\theta) \approx \left(\sum_{i=1}^n w_H^i \sum_{t = 0}^H \gamma^t \nabla_\theta \log \policy_\theta(\ba_t^i|\bs_t^i)  \hat{A}(\bs_t^i, \ba_t^i) \right) + \lambda \log \left( \sum_{i=1}^n w_H^i \right).
\]
It is easy to check that $\sum_i w_H^i \rightarrow 1$ as $n \rightarrow \infty$, meaning that this gradient estimator is consistent. However, with a finite number of samples, such an estimator automatically adjusts the policy $\policy_\theta$ to ensure that at least one sample has a high (unnormalized) importance weight. More recent deep reinforcement learning algorithms based on importance sampling often employ a sample-based KL-divergence regularizer~\citep{schulman2017proximal}, which has a functional form mathematically similar to this one when also utilizing an entropy regularizer on the policy $\policy_\theta$.

\subsection{Approximate Off-Policy Policy Gradients}
\label{sec:approx_grad}

The importance-weighted policy objective requires multiplying per-action importance weights over the time steps, which leads to very high variance. We can derive an approximate importance-sampled gradient by using the state distribution of the behavior policy, $\freq^\behavior(\bs)$, in place of that of the current policy, $\freq^\policy(\bs)$. This results in a biased gradient due to the mismatch in state distributions, but can provide reasonable learning performance in practice.
The corresponding objective, which we will denote $J_\behavior(\policy_\theta)$ to emphasize its dependence on the behavior policy's state distribution, is given by
\begin{equation*}
    J_\behavior(\policy_\theta) = \E_{\bs \sim \freq^\behavior(\bs)} \left[ V^\policy(\bs) \right].
\end{equation*}
Note that $J_\behavior(\policy_\theta)$ and $J(\policy_\theta)$ differ in the distribution of states under which the return is estimated ($\freq^\behavior(\bs)$ vs. $\freq^\policy(\bs)$), making $J_\behavior(\policy_\theta)$ a biased estimator for $J(\policy_\theta)$. This may lead to suboptimal solutions in certain cases (see \citet{imani2018off} for some examples). However, expectations under state distributions from $\freq^\behavior(\bs)$ can be calculated easily by sampling states from the dataset $\data$ in the offline case, removing the need for importance sampling.

Recent empirical work has found that this bias is acceptable in practice~\citep{fu2019diagnosing}. Differentiating the objective and applying a further approximation results in the \emph{off-policy policy gradient}~\citep{degris2012off}:
\begin{align*}
    \nabla_\theta J_\behavior(\pi_\theta) &= \E_{\bs \sim d^\behavior(\bs), \ba \sim \pi_\theta(\ba | \bs)}\left[ Q^{\pi_\theta}(\bs, \ba) \nabla_\theta \log \pi_\theta(\ba | \bs) +  \nabla_\theta Q^{\pi_\theta}(\bs, \ba) \right] \\
    &\approx \E_{\bs \sim d^\behavior(\bs), \ba \sim \pi_\theta(\ba | \bs)}\left[ Q^{\pi_\theta}(\bs, \ba) \nabla_\theta \log \pi_\theta(\ba | \bs) \right].
    \label{eq:off_policy_grad}
\end{align*}
\citet{degris2012off} show that under restrictive conditions, the approximate gradient preserves the local optima of $J_\behavior(\pi)$. This approximate gradient is used as a starting point in many widely used deep reinforcement learning algorithms~\citep{silver2014deterministic,lillicrap2015continuous,wang2016sample,gu2017interpolated}. 

To compute an estimate of the approximate gradient, we additionally need to estimate $Q^{\pi_\theta}(\bs, \ba)$ from the off-policy trajectories. Basic methods for doing this were discussed in Section~\ref{sec:rl_prelims}, and we defer a more in-depth discussion of offline state-action value function estimators to Section~\ref{sec:adp}. Some estimators use action samples, which required an importance weight to correct from $\behavior$ samples to $\policy_\theta$ samples. Further improvements can be obtained by introducing control variates and clipping importance weights to control variance~\citep{wang2016sample,espeholt2018impala}.

\subsection{Marginalized Importance Sampling}

If we would like to avoid the bias incurred by using the off-policy state distribution and the high variance from using per-action importance weighting, we can instead attempt to directly estimate the \emph{state-marginal importance ratio} $\rho^{\policy_\theta}(\bs) = \frac{d^{\policy_\theta}(\bs)}{d^\behavior(\bs)}$. An estimator using state marginal importance weights can be shown to be have no greater variance than using the product of per-action importance weights. However, computing this ratio exactly is generally intractable.  Only recently have practical methods for estimating the marginalized importance ratio been introduced~\citep{sutton2016emphatic,yu2015convergence,hallak2015emphatic,hallak2016generalized,hallak2017consistent,nachum2019dualdice,Zhang2020GenDICE:,nachum2019algaedice}.
We discuss some key aspects of these methods next.

Most of these methods utilize some form of dynamic programming to estimate the importance ratio $\rho^\policy$. Based on the form of the underlying Bellman equation used, we can classify them into two categories: methods that use a ``\textbf{forward}'' Bellman equation to estimate the importance ratios directly, and methods that use a ``\textbf{backward}'' Bellman equation on a functional that resembles a value function and then derive the importance ratios from the learned value functional.

\paragraph{Forward Bellman equation based approaches.} 

The state-marginal importance ratio, $\rho^\policy(\bs)$, satisfies a kind of ``forward'' Bellman equation:
\begin{equation}
    \forall \bs', ~~ \underbrace{d^\behavior(\bs')\rho^\pi(\bs')}_{\coloneqq (d^\behavior \circ \rho^\policy)(\bs')} = \underbrace{(1 - \gamma) \initstate(\bs') + \gamma \sum_{\bs, \ba} d^\behavior(\bs)\rho^\policy(\bs) \policy(\ba|\bs) T(\bs'|\bs, \ba)}_{\coloneqq (\bar{\bellman}^\pi \circ \rho^\policy)(\bs')} . 
    \label{eq:occupancy_bellman}
\end{equation}
This relationship can be leveraged to perform temporal difference updates to estimate the state-marginal importance ratio under the policy.

For example, when stochastic approximation is used, \citet{gelada2019off} use the following update rule in order to estimate $\rho^\policy(\bs')$ online:
\begin{equation}
    \hat{\rho}^\policy(\bs') \leftarrow \hat{\rho}^\policy(\bs') + \alpha \left[  (1 - \gamma) + \gamma \frac{\policy(\ba|\bs)}{\behavior(\ba|\bs)} \hat{\rho}^\policy(\bs) - \hat{\rho}^\policy(\bs') \right],
\end{equation}
with $\bs \sim d^\behavior(\bs), \ba \sim \behavior(\ba | \bs), \bs' \sim \transitions(\bs' | \bs, \ba)$.
Several additional techniques, including using TD($\lambda$) estimates and automatic adjustment of feature dimensions, have been used to stabilize learning. We refer the readers to \citet{hallak2017consistent} and \citet{gelada2019off} for a detailed discussion. \citet{gelada2019off} also discusses several practical tricks, such as soft-normalization and discounted evaluation, to adapt these methods to deep Q-learning settings, unlike prior work that mainly focuses on linear function approximation. \citet{wen2020vpm} view to problem from the lens of power iteration, and propose a variational power method approach to combine function approximation and power iteration to estimate $\rho^\policy$. \citet{imani2018off,zhang2019generalized} show that similar methods can be used to estimate the off-policy policy gradient and thus be used in an off-policy actor critic method. 

Alternatively, \citet{liu2018breaking} propose to use an adversarial approach to obtain the state-marginal importance ratios. From Eq.~\ref{eq:occupancy_bellman}, they derive a functional
\begin{equation}
    L(\rho, f) = \gamma \E_{\bs, \ba, \bs' \sim \data}\left[\left(\rho(\bs) \frac{\policy(\ba | \bs)}{\behavior(\ba | \bs)}  - \rho(\bs')\right) f(\bs')\right] + (1 - \gamma) \E_{\bs_0 \sim d_0} [(1 - \rho(\bs)) f(\bs)]    
\end{equation}
such that $L(\rho, f) = 0, \forall ~f$ if and only if $\rho = \rho^\policy$. Therefore, we can learn $\rho$ by minimizing a worst-case estimate of $L(\rho, f)^2$, by solving an adversarial, saddle-point optimization: $\min_{\rho} \max_{f} L(\rho, f)^2$. Recent work~\citep{mousavi2020black,kallus2019efficiently,kallus2019intrinsically,tang2019doubly,uehara2019minimax} has refined this approach, in particular removing the need to have access to $\behavior$. Once $\rho^*$ is obtained, \citet{liu2019off} use this estimator to compute the off-policy policy gradient.  

\citet{Zhang2020GenDICE:} present another off-policy evaluation method that computes the importance ratio for the state-action marginal, $\rho^\policy(\bs, \ba) := \frac{d^\policy(\bs, \ba)}{d^\behavior(\bs, \ba)}$, by directly optimizing a variant of the Bellman residual error corresponding to a modified forward Bellman equation, that includes actions, shown in Equation~\ref{eq:occupancy_bellman_with_actions}. 
\begin{equation}
    \underbrace{d^\behavior(\bs', \ba')\rho^\pi(\bs', \ba')}_{\coloneqq (d^\behavior \circ \rho^\policy)(\bs', \ba')}\!=\!\underbrace{(1 - \gamma) \initstate(\bs') \policy(\ba'|\bs')\!+\!\gamma \sum_{\bs, \ba} d^\behavior(\bs, \ba)\rho^\policy(\bs, \ba) \policy(\ba|\bs) T(\bs'|\bs, \ba)}_{\coloneqq (\bar{\bellman}^\pi \circ \rho^\policy)(\bs', \ba')}. 
    \label{eq:occupancy_bellman_with_actions}
\end{equation}
Their approach can be derived by applying a divergence metric (such as an f-divergence, which we will review in Section~\ref{sec:adp}) between the two sides of the modified forward Bellman equation in Equation~\ref{eq:occupancy_bellman_with_actions}, while additionally constraining the importance ratio, $\rho^\policy(\bs, \ba)$, to integrate to $1$ in expectation over the dataset, $\mathcal{D}$, to prevent degenerate solutions, as follows
\begin{equation}
    \label{eqn:gendice}
    \min_{\rho^\policy}~~ D_f\left( \left(\bar{\bellman}^\policy \circ \rho^\policy\right) (\bs, \ba), \left( d^\behavior \circ \rho^\policy  \right)(\bs, \ba) \right) ~~~ \text{s.t.} ~~~ \E_{\bs, \ba, \bs' \sim \mathcal{D}}[\rho^\policy(\bs, \ba)] = 1.
\end{equation}
They further apply tricks inspired from dual embeddings~\citep{dai2016learning} to make the objective tractable, and to avoid the bias caused due to sampled estimates. We refer the readers to \citet{Zhang2020GenDICE:} for further discussion.

\citet{zhang2020gradientdice} show that primal-dual solvers may not be able to solve Eqn.~\ref{eqn:gendice}, and modify the objective in  by replacing the f-divergence with a norm induced by $1/d^\behavior$. This creates an optimization problem that is provably convergent under linear function approximation.
\begin{equation}
    \label{eqn:gradientdice}
    \min_{\rho^\policy}~~ \frac{1}{2}|| \left( \left(\bar{\bellman}^\policy \circ \rho^\policy\right) (\bs, \ba) - \left( d^\behavior \circ \rho^\policy  \right)(\bs, \ba) \right ||^2_{(d^\behavior)^{-1}} + \frac{\lambda}{2} (\E_{\bs, \ba, \bs' \sim \mathcal{D}}[\rho^\policy(\bs, \ba)] - 1)^2.
\end{equation}

\paragraph{Backward Bellman equation based approaches via convex duality.} Finally, we discuss methods for off-policy evaluation and improvement that utilize a backward Bellman equation -- giving rise to a value-function like functional -- by exploiting convex duality. Because these methods start from an optimization perspective, they can bring to bear the mature tools of convex optimization and online learning. \citet{lee2018stochastic} extend a line of work applying to convex optimization techniques to policy optimization~\citep{chen2016stochastic,wang2016online,dai2017boosting,dai2017sbeed} to the off-policy setting. They prove sample complexity bounds in the off-policy setting, however, extending these results to practical deep RL settings has proven challenging.

\citet{nachum2019dualdice} develop similar ideas for off-policy  evaluation. The key idea is to devise a convex optimization problem with $\rho^\policy$ as its optimal solution. The chosen optimization problem is
\begin{equation}
\label{eqn:dualdice_obj}
    \rho^\policy = \arg\min_{x : \mathcal{S} \times \mathcal{A} \rightarrow \mathbb{R}}~~ \frac{1}{2} \E_{\bs, \ba, \bs' \sim \mathcal{D}}\left[x(\bs, \ba)^2 \right] - \E_{\bs \sim d^\policy(\bs), \ba \sim \pi(\ba|\bs)}\left[x(\bs, \ba)\right].
\end{equation}
Unfortunately, this objective requires samples from the on-policystate-marginal  distribution, $d^\policy(\bs)$. The key insight is a change of variables, \mbox{$x(s, a) = \nu(\bs, \ba) - \E_{\bs' \sim T(\bs'|\bs, \ba), \ba' \sim \policy(\ba'|\bs')} [\nu(\bs', \ba')]$} and introduce the variable $\nu(\bs, \ba)$ to simplify the expression in Equation~\ref{eqn:dualdice_obj}, and obtain the ``backward'' dynamic programming procedure shown in Equation~\ref{eqn:dualdice}. For brevity, we define a modified Bellman operator, $\tilde{\bellman}^\pi \nu(\bs, \ba) := \E_{\bs' \sim T(\bs'|\bs, \ba), \ba' \sim \pi(\ba'|\bs')} [\nu(\bs', \ba')]$, that is similar to the expression for $\bellman^\policy$ without the reward term $r(s, a)$.  
\begin{equation}
\label{eqn:dualdice}
    \min_{\nu: \mathcal{S} \times \mathcal{A} \rightarrow \mathbb{R}}~~ \frac{1}{2} \E_{\bs, \ba, \bs' \sim \mathcal{D}}\left[\left(\nu(\bs, \ba) - \tilde{\bellman}^\pi \nu(\bs, \ba)\right)^2 \right] - \E_{\bs_0 \sim d_0(\bs_0), \ba \sim \pi(\ba|\bs_0)}\left[\nu(\bs_0, \ba)\right].
\end{equation}
Remarkably, Equation~\ref{eqn:dualdice} does not require on-policy samples to evaluate. Given an optimal solution for the objective in Equation~\ref{eqn:dualdice}, denoted as $\nu^*$, we can obtain the density ratio, $\rho^\policy$, using the relation, $\rho^\policy(\bs, \ba) = \nu^*(\bs, \ba) - \tilde{\bellman}^\pi \nu^*(\bs, \ba)$. The density ratio can then be used for off-policy evaluation and improvement.

\citet{nachum2019algaedice,nachum2020reinforcement} build on a similar framework to devise an off-policy RL algorithm. The key idea is to obtain an estimate of the on-policy policy gradient for a state-marginal \textit{regularized} RL objective by solving an optimization problem. The regularizer applied in this family of methods is the f-divergence between the state(-action) marginal of the learned policy and the state-action marginal of the dataset. We will cover f-divergences in detail in Section~\ref{sec:adp}. 
The f-divergence regularized RL problem, with a tradeoff factor, $\alpha$, is given by:
\begin{equation}
\label{eqn:regularized_rl}
    \max_{\policy}~~ \E_{\bs \sim d^\policy(\bs), \ba \sim \policy(\cdot|\bs)}\left[r(\bs, \ba) \right] - \alpha D_f(d^\policy(\bs, \ba), d^\behavior(\bs, \ba)).
\end{equation}
By exploiting the variational (dual) form of the f-divergence shown below
\begin{equation}
\label{eqn:variational_f_div_section3}
    D_f(p, q) = \max_{x : \mathcal{S} \times \mathcal{A} \rightarrow \mathbb{R}} \left(\E_{\by \sim p(\by)}[x(\by)] - \E_{\by \sim q(\by)}[f^*(x(\by))]\right),
\end{equation}
and then applying a change of variables from $x$ to $Q$ (c.f.,~\citep{nachum2019dualdice}) where $Q$ satisfies   
$Q(\bs, \ba) = \E_{\bs'\sim\transitions(\bs, \ba),\ba'\sim\policy(\ba'|\bs')} \left[ r(\bs, \ba) - \alpha x(\bs, \ba) + \gamma Q(\bs', \ba') \right]$, we obtain a saddle-point optimization problem for optimizing the regularized RL objective,
\begin{align*}
    \max_{\policy}~ \min_{Q}& ~~ L(Q, \behavior, \pi) \coloneqq \E_{\bs_0 \sim d_0(\bs_0), \ba \sim \policy(\cdot|\bs_0)}\left[Q(\bs_0, \ba) \right] \\
    &~~~~~~~~+ \alpha \E_{\bs, \ba \sim d^\behavior(\bs, \ba)}\left[f^*\left(\frac{r(\bs, \ba) + \gamma \E_{\bs'\sim\transitions(\bs, \ba),\ba'\sim\policy(\ba'|\bs')}[Q(\bs', \ba')] - Q(\bs, \ba)}{\alpha}\right)\right].
\end{align*}
When $f(x) = x^2$, $f^*(x) = x^2$ and this objective reduces to applying a regular actor-critic algorithm as discussed in Section~\ref{sec:rl_prelims} along with an additional term minimizing Q-values at the initial state $\bs_0$. 

It can also be shown that the derivative with respect to the policy of $L(Q^*, \behavior, \policy)$, at the optimal Q-function, $Q^*$, is precisely equal to the on-policy policy gradient in the regularized policy gradient problem:
\begin{equation}
    \frac{\partial}{\partial \policy} L(Q^*, \behavior, \pi) = \E_{\bs \sim d^\policy(\bs), \ba \sim \pi(\cdot|\bs)}\left[ \tilde{Q}^\policy(\bs, \ba) \cdot \nabla_\pi \log \policy(\ba|\bs)\right],
\end{equation}
where $\tilde{Q}^\policy$ is the action-value function corresponding to the regularized RL problem.

Finally, we note that this is a rapidly developing area and recent results suggest that many of these methods can be unified under a single framework~\citep{tang2019doubly,nachum2020reinforcement}.

\subsection{Challenges and Open Problems}
\label{sec:pg_challenges}

The methods discussed in this section utilize some form of importance sampling to either estimate the return of the current policy $\policy_\theta$, or estimate the gradient of this return. The policy improvement methods discussed in this section have been developed primarily for a classic off-policy learning setting, where additional data is collected online, but previously collected data is reused to improve efficiency. To the best of our knowledge, such methods have not generally been applied in the offline setting, with the exception of the evaluation and high-confidence improvement techniques in Section~\ref{sec:ope}.

Applying such methods in the fully offline setting poses a number of major challenges. First, importance sampling already suffer from high variance, and this variance increases dramatically in the sequential setting, since the importance weights at successive time steps are multiplied together (see, e.g., Equation~(\ref{eq:return_is})), resulting in exponential blowup. Approximate and marginalized importance sampling methods alleviate this issue to some degree by avoiding multiplication of importance weights over multiple time steps, but the fundamental issue still remains: when the behavior policy $\behavior$ is too different from the current learned policy $\policy_\theta$, the importance weights will become degenerate, and any estimate of the return or the gradient will have too much variance to be usable, especially in high-dimensional state and action spaces, or for long-horizon problems. For this reason, importance-sampled estimators are most suitable in the case where the policy only deviates by a limited amount from the behavior policy. In the classic off-policy setting, this is generally the case, since new trajectories are repeatedly collected and added to the dataset using the latest policy, but in the offline setting this is generally not the case. Thus, the maximum improvement that can be reliably obtained via importance sampling is limited by (i) the suboptimality of the behavior policy; (ii) the dimensionality of the state and action space; (iii) the effective horizon of the task. The second challenge is that the most effective of these off-policy policy gradient methods either requires estimating the value function, or the state-marginal density ratios via dynamic programming. As several prior works have shown, and as we will review in Section~\ref{sec:adp}, dynamic programming methods suffer from issues pertaining to out-of-distribution queries in completely offline settings, making it hard to stably learn the value function without additional corrections. This problem is not as severe in the classic off-policy setting, which allows additional data collection. Nonetheless, as we will discussion in Section~\ref{sec:applications}, a number of prior applications have effectively utilized importance sampling in an offline setting.

\section{Offline Reinforcement Learning via Dynamic Programming}
\label{sec:adp}

Dynamic programming methods, such as Q-learning algorithms, in principle can offer a more attractive option for offline reinforcement learning as compared to pure policy gradients. As discussed in Section~\ref{sec:rl_prelims}, dynamic programming methods aim to learn a state or state-action value function, and then either use this value function to directly recover the optimal policy or, as in the case of actor-critic methods, use this value function to estimate a gradient for the expected returns of a policy. Basic offline dynamic programming algorithms can be constructed on the basis of fitted Q-learning methods~\citep{ernst2005tree,riedmiller2005neural,hafner2011reinforcement}, as well as policy iteration methods~\citep{sb-irl-98}. The generic Q-learning and actor-critic algorithms presented in Algorithm~\ref{alg:qlearning} and Algorithm~\ref{alg:actorcritic} in Section~\ref{sec:rl_prelims} can in principle be utilized as offline reinforcement learning, simply by setting the number of collection steps $S$ to zero, and initializing the buffer to be non-empty. Such algorithms can be viable for offline RL, and indeed have been used as such even in recent deep reinforcement learning methods. We will discuss an older class of sucm methods, based on linear function approximation, in Section~\ref{sec:basic_offline_adp}, but such techniques have also been used effectively with deep neural network function approximators. For example, \citet{kalashnikov2018qtopt} describes a Q-learning algorithm called QT-Opt that was able to learn effective vision-based robotic grasping strategies from about 500,000 grasping trials logged over the course of previous experiments, but observes that additional online fine-tuning still improved the performance of the policy considerably over the one trained purely from logged data. Some prior works on offline RL have also noted that, for some types of datasets, conventional dynamic programming algorithms, such as deep Q-learning or deterministic actor-critic algorithms, can in fact work reasonably well~\citep{agarwal2019optimistic}.

However, as we will discuss in Section~\ref{sec:challenges_adp}, such methods suffer from a number of issues when all online collection is halted, and only offline data is used. These issues essentially amount to distributional shift, as discussed in Section~\ref{sec:basic_challenges}. Solutions to this issue can be broadly grouped into two categories: \textbf{policy constraint} methods, discussed in Section~\ref{sec:policy_constraints}, which constrain the learned policy $\policy$ to lie close to the behavior policy $\behavior$, thus mitigating distributional shift, and \textbf{uncertainty-based} methods, discussed in Section~\ref{sec:adp_uncertainty}, which attempt to estimate the epistemic uncertainty of Q-values, and then utilize this uncertainty to detect distributional shift. We will discuss both classes of distributional shift corrections, and then conclude with perspectives on remaining challenges and open problems in Section~\ref{sec:adp_challenges}.

\subsection{Off-Policy Value Function Estimation with Linear Value Functions}
\label{sec:basic_offline_adp}

We first discuss standard offline reinforcement learning methods based on value function and policy estimation with linear function approximators, which do not by themselves provide any mitigate for distributional shift, but can work well when good linear features are available. While modern deep reinforcement learning methods generally eschew linear features in favor of non-linear neural network function approximators, linear methods constitute an important class of offline reinforcement learning algorithms in the literature~\citep{lange2012batch,lagoudakis2003least}. We begin with algorithms that use a linear function to approximate the Q-function, such that $Q_\phi \approx \qfeat(s, a)^T \phi$, where $\qfeat(s, a) \in \mathbb{R}^d$ is a feature vector associated with a state-action pair. 
This parametric Q-function can then be estimated for a given policy $\policy(\ba|\bs)$ using data from a different behavior policy $\behavior(\ba|\bs)$, with state visitation frequency $\freq^{\behavior}(\bs)$. As discussed in Section~\ref{sec:rl_prelims}, the Q-function $Q^\policy$ for a given policy $\policy(\ba|\bs)$ must satisfy the Bellman equation, given in full in Equation~(\ref{eq:qeq}), and written in Bellman operator notation as $\vec{Q^\policy} = \bellman^\policy \vec{Q^\policy}$.

When linear function approximation is used to represent the Q-function, the Q-function for a policy can be represented as the solution of a linear system, and hence can be computed via least squares minimization, since the Bellman operator $\bellman^\policy$ is linear. This provides a convenient way to compute $Q^\policy$ directly in closed form, as compared to using gradient descent in Algorithm~\ref{alg:actorcritic}. The resulting Q-value estimates can then be used to improve the policy.
We start with a discussion of different ways of solving the linear system for computing $Q^\policy$. To recap, the Q-function is a linear function on a feature vector $\qfeat(\bs,\ba)$, which we can express in tabular form as $\Qfeat \in \mathbb{R}^{|S||A| \times d}$, such that $\vec{Q}_\phi = \Qfeat \phi$. Multiple procedures that can be used to obtain a closed form expression for the optimal $\phi$ for a given policy $\policy$ and, as discussed in by \citet{sutton2009fastgtd} and \citet{lagoudakis2003least}, these procedures may each yield different solutions under function approximation. We first summarize two methods for selecting $\phi$ to approximate $Q^\policy$ for a given policy $\policy$, and then discuss how to utilize these methods in a complete reinforcement learning method.

\paragraph{Bellman residual minimization.} The first approach selects $\phi$ such that the resulting linear Q-function satisfies the Bellman equation as closely as possible in terms of squared error, which can be obtained via the least squares solution. To derive the corresponding least squares problem, we first write the Bellman equation in terms of the Bellman operator, and expand it using the vectorized expression for the reward function, $\vec{R}$, and a linear operator representing the dynamics and policy backup, which we denote as $P^\policy$, such that $(P^\policy \vec{Q}) (\bs,\ba) = \E_{\bs' \sim \transitions(\bs'|\bs,\ba), \ba'\sim\policy(\ba'|\bs')}[Q(\bs',\ba')]$. The corresponding expression of the Bellman equation is given by
\begin{equation*}
\Qfeat \phi \approx \bellman^\policy \Qfeat \phi = \vec{R} + \discount P^\policy \Qfeat \phi \implies
\left(\Qfeat - \discount P^\policy \Qfeat \right) \phi \approx \vec{R}.
\end{equation*}
Writing out the least squares solution using the normal equations, we obtain
\begin{equation*}
\phi = \left( (\Qfeat - \discount P^\policy \Qfeat)^T (\Qfeat - \discount P^\policy \Qfeat) \right)^{-1} (\Qfeat - \discount P^\policy \Qfeat)^T \vec{R}.
\end{equation*}
This expression minimizes the $\ell_2$ norm of the Bellman residual (the squared difference between the left-hand side and right-hand side of the Bellman equation), and is referred to as the Bellman residual minimizing solution.

\paragraph{Least-squares fixed point approximation.} An alternate approach is use projected fixed-point iteration, rather than direct minimization of the Bellman error, which gives rise to the least-squares fixed point approximation. In this approach, instead of minimizing the squared difference between the left-hand side and right-hand side of the Bellman equation, we instead attempt to iterate the Bellman operator to convergence. In the tabular case, as discussed in Section~\ref{sec:rl_prelims}, we know that iterating $\vec{Q}_{k+1} \leftarrow \bellman^\policy \vec{Q}_k$ converges, as $k \rightarrow \infty$, to $\vec{Q}^\policy$. In the case where we use function approximation to represent $\vec{Q}_k$, we cannot set $\vec{Q}_{k+1}$ to $\bellman^\policy \vec{Q}_k$ precisely, because there may not be any value of the weights $\phi$ that represent $\bellman^\policy \vec{Q}_k$ exactly. We therefore must employ a \emph{projected} fixed point iteration method, where each iterate $\bellman^\policy \vec{Q}_k = \bellman^\policy\Qfeat\phi_k$ is projected onto the span of $\Qfeat$ to obtain $\phi_{k+1}$. We can express this projection via a projection operator, $\proj$, such that the projected Bellman iteration is given by $\vec{Q}_{k+1} = \proj\bellman^\policy \vec{Q}_k$. We can obtain a solution for this operator by solving the normal equation, and observe that $\proj = \Qfeat (\Qfeat^T \Qfeat)^{-1} \Qfeat^T$. We can expand out the projected Bellman iteration expression as:
\begin{align}
\vec{Q}_{k+1} &= \Qfeat (\Qfeat^T \Qfeat)^{-1} \Qfeat^T(\vec{R} + \discount P^\policy \vec{Q}_k) \nonumber \\
\Qfeat \phi_{k+1} &= \Qfeat (\Qfeat^T \Qfeat)^{-1} \Qfeat^T(\vec{R} + \discount P^\policy \Qfeat \phi_k) \nonumber \\
\phi_{k+1} &= (\Qfeat^T \Qfeat)^{-1} \Qfeat^T(\vec{R} + \discount P^\policy \Qfeat \phi_k). \label{eqn:w_pi_equation}
\end{align}
By repeatedly applying this recurrence, we can obtain the fixed point of the projected Bellman operator as $k \rightarrow \infty$~\citep{sutton2009fastgtd}.

Both methods have been used in the literature, and there is no clear consensus on which approach is preferred, though they yield different solutions in general when the true Q-function $Q^\policy$ is not in the span of $\Qfeat$ (i.e., cannot be represented by any parameter vector $\phi$)~\citep{lagoudakis2003least}. We might at first surmise that a good linear fitting method should find the Q-function $\Qfeat \phi$ that corresponds to a least-squares projection of the true $\vec{Q}^\policy$ onto the hyperplane defined by $\Qfeat$. Unfortunately, \emph{neither} the Bellman residual minimization method nor the least-squares fixed point method in general obtains this solution. The Bellman residual minimization does not in general produce a fixed point of \emph{either} the Bellman operator or the projected Bellman operator. However, the solution obtained via Bellman residual minimization may be closer to the true Q-function in terms of $\ell_2$ distance, since it is explicitly obtained by minimizing Bellman residual error. Least-squares fixed point iteration obtains a Q-function that is a fixed point of the projected Bellman operator, but may be arbitrarily suboptimal. However, least-squares fixed point iteration can learn solutions that lead, empirically, to better-performing policies \citep{sutton2009fastgtd,lagoudakis2003least}. In general, there are no theoretical arguments that justify the superiority of one approach over the other. In practice, least-squares fixed-point iteration can give rise to more effective policies as compared to the Bellman residual, while the Bellman residual minimization approach can be more stable and predictable~\citep{lagoudakis2003least}.

\paragraph{Least squares temporal difference Q-learning (LSTD-Q).} Now that we have covered different approaches to solve a linear system of equations to obtain an approximation to $Q^\policy$, we describe LSTD-Q, a method for estimating $Q^\policy$ from a static dataset, directly from samples. This method incrementally estimates the terms $\Qfeat^T (\Qfeat - \discount P^{\policy} \Qfeat)$ and $\Qfeat^T \vec{R}$, which appear in Equation~(\ref{eqn:w_pi_equation}), and then inverts this sampled estimate to then obtain $\phi$. We defer further discussion on LSTD-Q to \citet{lagoudakis2003least}, which also describes several computationally efficient implementations of the LSTD-Q algorithm. Note that the LSTD-Q algorithm is not directly applicable to estimating $Q^*$, the optimal Q-function, since the optimal Bellman equation for $Q^*$ is not linear due to the maximization, and thus cannot be solved in closed form.

\paragraph{Least squares policy iteration (LSPI).} Finally, we discuss least-squares policy iteration (LSPI), a classical offline reinforcement learning algorithm that performs approximate policy iteration (see discussion in Section~\ref{sec:rl_prelims}) using a linear approximation to the Q-function. LSPI uses LSTD-Q as an intermediate sub-routine to perform approximate policy evaluation, thereby obtaining an estimate for $Q^\policy$, denoted $Q_\phi$. Then, the algorithm sets the next policy iterate to be equal to the greedy policy corresponding to the approximate $Q_\phi$, such that $\policy_{k+1}(\ba|\bs) = \delta(\ba = \arg\max_\ba Q_\phi(\bs,\ba))$.
An important and useful characteristic of LSPI is that it does not require a separate approximate representation for the policy when the actions are discrete, and hence removes any source of error that arises due to function approximation in the actor in actor-critic methods. However, when the action space is continuous, a policy gradient method similar to the generic actor-critic algorithm in Algorithm~\ref{alg:actorcritic} can be used to optimize a parametric policy.

\subsection{Distributional Shift in Offline Reinforcement Learning via Dynamic Programming}
\label{sec:challenges_adp}

Both the linear and non-linear dynamic programming methods discussed so far, in Section~\ref{sec:rl_prelims} and Section~\ref{sec:basic_offline_adp} above, can \emph{in principle} learn from offline data, collected according to a different (unknown) behavioral policy $\behavior(\ba|\bs)$, with state visitation frequency $\freq^{\behavior}(\bs)$. However, in practice, these procedures can result in very poor performance, due to the distributional shift issues alluded to in Section~\ref{sec:basic_challenges}.

Distributional shift affects offline reinforcement learning via dynamic programming both at test time and at training time. At test time, the state visitation frequency $\freq^{\policy}(\bs)$ induced by a policy trained with offline RL will differ systematically from the state visitation frequency of the training data, $\freq^{\behavior}(\bs)$. This means that, as in the case of policy gradients, a policy trained via an actor-critic method may produce unexpected and erroneous actions in out-of-distribution states $\bs \sim \freq^{\behavior}(\bs)$, as can the implicit greedy policy induced by a Q-function in a Q-learning method. One way to mitigate this distributional shift is to limit how much the learned policy can diverge from the behavior policy. For example, by constraining $\policy(\ba|\bs)$ such that $\kl(\policy(\ba|\bs) \| \behavior(\ba|\bs)) \leq \epsilon$, we can bound $\kl(\freq^{\policy}(\bs)\|\freq^{\behavior}(\bs))$ by $\delta$, which is $O(\epsilon/(1-\gamma)^2)$~\citep{schulman2015trust}. This bound is very loose in practice, but nonetheless suggests that the effects of state distribution shift can potentially be mitigated by bounding how much the learned policy can deviate from the behavior policy that collected the offline training data. This may come at a substantial cost in final performance however, as the behavior policy -- and any policy that is close to it -- may be much worse than the best policy that can be learned from the offline data.

It should be noted that, for the algorithms discussed previously, state distribution shift affects test-time performance, but has no effect on training, since neither the policy nor the Q-function is ever evaluated at any state that was not sampled from $\freq^{\behavior}(\bs)$. However, the training process \emph{is} affected by \emph{action} distribution shift, because the target values for the Bellman backups in Equation~(\ref{eq:qeq}) depend on $\ba_{t+1} \sim \policy(\ba_{t+1} | \bs_{t+1})$. While this source of distribution shift may at first seem insignificant, it in fact presents one of the largest obstacles for practical application of approximate dynamic programming methods to offline reinforcement learning. Since computing the target values in Equation~(\ref{eq:qeq}) requires evaluating $Q^\policy(\bs_{t+1},\ba_{t+1})$, where $\ba_{t+1} \sim \policy(\ba_{t+1} | \bs_{t+1})$, the accuracy of the Q-function regression targets depends on the estimate of the Q-value for actions that are outside of the distribution of actions that the Q-function was ever trained on. When $\policy(\ba | \bs)$ differs substantially from $\behavior(\ba|\bs)$, this discrepancy can result in highly erroneous target Q-values. This issue is further exacerbated by the fact that $\policy(\ba | \bs)$ is explicitly optimized to maximize $\E_{\ba\sim\policy(\ba|\bs)}[Q^\policy(\bs,\ba)]$. This means that, if the policy can produce out-of-distribution actions for which the learned Q-function erroneously produces excessively large values, it will learn to do so. This is true whether the policy is represented explicitly, as in actor-critic algorithms, or implicitly, as the greedy policy $\policy(\ba|\bs) = \delta(\ba = \arg\max_{\ba'} Q^\policy(\bs,\ba'))$. In standard online reinforcement learning, such issues are corrected naturally when the policy acts in the environment, attempting the transitions it (erroneously) believes to be good, and observing that in fact they are not. However, in the offline setting, the policy cannot correct such over-optimistic Q-values, and these errors accumulate over each iteration of training, resulting in arbitrarily poor final results.

\begin{wrapfigure}{r}{0.5\textwidth}
\begin{center}
\vspace{-0.1in}
    \includegraphics[width=0.48\linewidth]{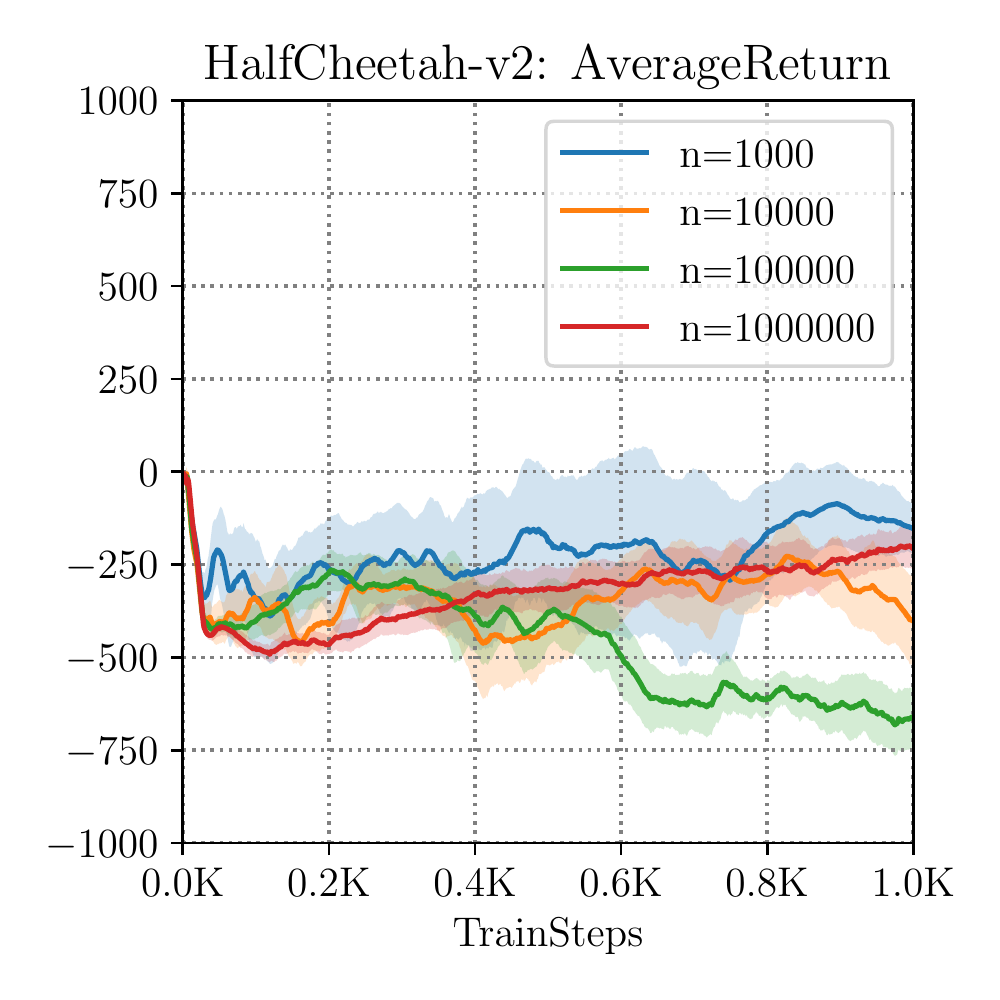}
    ~
    \includegraphics[width=0.48\linewidth]{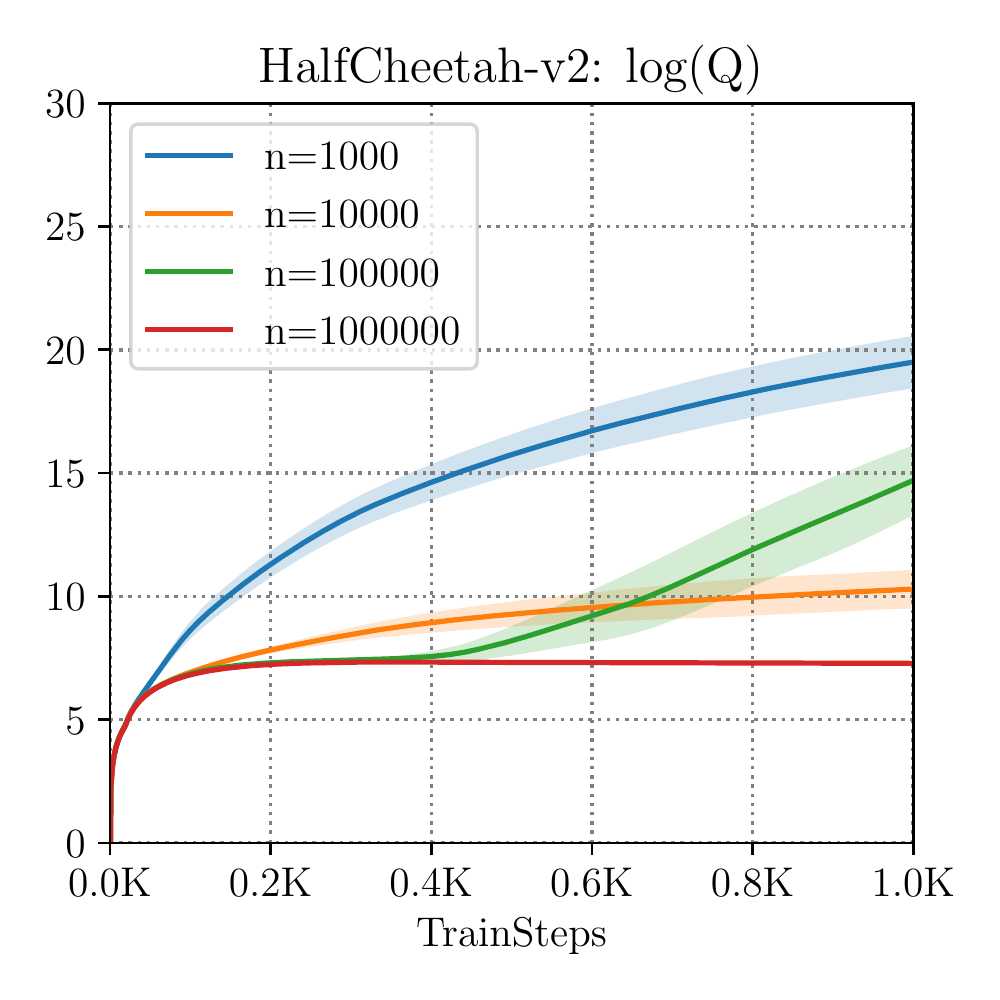}
  \end{center}
 \vspace{-10pt}
  \caption{Performance of SAC~\citep{sac}, an actor-critic method, on the HalfCheetah-v2 task in the offline setting, showing return as a function of gradient steps (left) and average learned Q-values on a log scale (right), for different numbers of training points ($n$). Note that an increase the number of samples does not generally prevent the ``unlearning effect,'' indicating that it is distinct from overfitting. Figure from \citet{kumar2019stabilizing}.} 
 \vspace{-15pt}
 \label{fig:bear_overfitting}
\end{wrapfigure}
This problem manifests itself in practice as an ``unlearning'' effect when running offline RL via dynamic programming. The experiments in Figure~\ref{fig:bear_overfitting}, from \citet{kumar2019stabilizing}, show this unlearning effect. The horizontal axis corresponds to the number of gradient updates to the Q-function, and the vertical axis shows the actual return obtained by running the greedy policy for the learned Q-function. The learning curve resembles what we might expect in the case of overfitting: the return first improves, and then sharply falls as training progresses. However, this ``overfitting'' effect remains even as we increase the size of the dataset, suggesting that the phenomenon is distinct from classic statistical overfitting. As the Q-function is trained longer and longer, the target values become more and more erroneous, and the entire Q-function degrades.

Thus, to effectively implement offline reinforcement learning via dynamic programming, it is crucial to address this out-of-distribution action problem. Previous works have observed several variants of this problem. \citet{fujimoto2018off} noted that restricting Q-value evaluation only to actions in the dataset avoids erroneous Q-values due to extrapolation error. \citet{kumar2019stabilizing} described the out-of-distribution action problem in terms of distributional shift, which suggests a less restrictive solution based on limiting distributional shift, rather than simply constraining to previously seen actions. A number of more recent works also observe that a variety of constraints can be used to mitigate action distribution shift~\citep{wu2019behavior}. We provide a unified view of such ``policy constraint'' methods in the following section.

\vspace{-5pt}
\subsection{Policy Constraints for Off-Policy Evaluation and Improvement}
\label{sec:policy_constraints}

The basic idea behind \textbf{policy constraint} methods for offline reinforcement learning via dynamic programming is to ensure, explicitly or implicitly, that regardless of the target values in Equation~(\ref{eq:qeq}), the distribution over actions under which we compute the target value, $\policy(\ba'|\bs')$, is ``close'' to the behavior distribution $\behavior(\ba'|\bs')$. This ensures that the Q-function is never queried on out-of-distribution actions, which may introduce errors into the computation of the Bellman operator. If all states and actions fed into the Q-function for target value calculations are always in-distribution with respect to the Q-function training set, errors in the Q-function should not accumulate, and standard generalization results from supervised learning should apply. Since the Q-function is evaluated on the same states as the ones on which it is trained, only the action inputs can be out of distribution, when we compute $\E_{\ba'\sim\policy(\ba'|\bs')}[Q(\bs',\ba')]$, and therefore it is sufficient to keep $\policy(\ba'|\bs')$ close to $\behavior(\ba'|\bs')$. Of course, in practice, the distributions do need to deviate in order for the learned policy to improve over the behavior policy, but by keeping this deviation small, errors due to out-of-distribution action inputs can be kept under control. The different methods in this family differ in terms of the probability metric used to define ``closeness'' and how this constraint is actually introduced and enforced. We can broadly categorize these methods along these two axes. We will discuss \textbf{explicit $f$-divergence constraints}, which directly add a constraint to the actor update to keep the policy $\policy$ close to $\behavior$ in terms of an $f$-divergence (typically the KL-divergence), \textbf{implicit $f$-divergence constraints}, which utilize an actor update that, by construction, keeps $\policy$ close to $\behavior$, and \textbf{integral probability metric (IPM) constraints}, which can express constraints with more favorable theoretical and empirical properties for offline RL. Furthermore, the constraints can be enforced either as direct \textbf{policy constraints} on the actor update, or via a \textbf{policy penalty} added to the reward function or target Q-value.

Formally, we can express the family of policy iteration methods with policy constraints as a fixed point iteration involving iterative (partial) optimization of the following objectives:
\begin{align*}
\hat{Q}^\policy_{k+1} &\leftarrow \arg\min_Q \E_{(\bs,\ba,\bs') \sim \data} \left[
\left(
Q(\bs,\ba) - \left(
r(\bs,\ba) + \discount \E_{\ba'\sim \policy_k(\ba'|\bs')}[\hat{Q}^\policy_k(\bs',\ba')]
\right)
\right)^2
\right] \\
\policy_{k+1} &\leftarrow \arg\max_\policy \E_{\bs \sim \data}\left[
\E_{\ba \sim \policy(\ba|\bs)}[\hat{Q}^\policy_{k+1}(\bs,\ba)]
\right] \text{ s.t. } D(\policy, \behavior) \leq \epsilon.
\end{align*}
When the $\min$ and $\max$ optimizations are not performed to convergence, but instead for a limited number of gradient steps, we recover the general actor-critic method in Algorithm~\ref{alg:actorcritic}, with the exception of the constraint $D(\policy, \behavior) \leq \epsilon$ on the policy update.
A number of prior methods instantiate this approach with different choices of $D$~\citep{kumar2019stabilizing,fujimoto2018off,siegel2020keep}. We will refer to this class of algorithms as \textbf{policy constraint} methods.

In \textbf{policy penalty} methods, the actor-critic algorithm is modified to incorporate the constraint into the Q-values, which forces the policy to not only avoid deviating from $\behavior(\ba|\bs)$ at each state, but to also avoid actions that may lead to higher deviation from $\behavior(\ba|\bs)$ at future time steps. This can be accomplished by adding a penalty term $\alpha D(\policy(\cdot | \bs), \behavior(\cdot | \bs))$ to the reward function $\reward(\bs,\ba)$ leading to the modified reward function $\bar{\reward}(\bs,\ba) = \reward(\bs,\ba) - \alpha D(\policy(\cdot | \bs), \behavior(\cdot | \bs))$. The most well-known formulation of policy penalty methods stems from the linearly solvable MDP framework~\citep{todorov_lmdp}, or equivalently the control as inference framework~\citep{levine2018reinforcement}, where $D$ is chosen to be the KL-divergence. An equivalent formulation adds the penalty term $\alpha D(\policy(\cdot | \bs), \behavior(\cdot | \bs))$ directly to the target Q-values and the actor objective~\citep{wu2019behavior,jaques2019way}, resulting in the following algorithm:
\begin{align*}
\hat{Q}^\policy_{k+1} &\leftarrow \arg\min_Q\\ &\E_{(\bs,\ba,\bs') \sim \data}\! \left[
\left(
Q(\bs,\ba)\! -\! \left(
r(\bs,\ba)\! +\! \discount \E_{\ba'\sim \policy_k(\ba'|\bs')}[\hat{Q}^\policy_k(\bs',\ba')]\! -\! \alpha \gamma \! D(\policy_k(\cdot|\bs'), \behavior(\cdot|\bs'))
\right)
\right)^2
\right] \\
\policy_{k+1} &\leftarrow \arg\max_\policy \E_{\bs \sim \data}\left[
\E_{\ba \sim \policy(\ba|\bs)}[\hat{Q}^\policy_{k+1}(\bs,\ba)] -  \alpha D(\policy(\cdot|\bs), \behavior(\cdot|\bs)) 
\right].
\end{align*}
While the basic recipe for policy constraint and policy penalty methods is similar, the particular choice of how the constraints are defined and how they are enforced can make a significant difference in practice. We will discuss these choices next, as well as their tradeoffs.

\paragraph{Explicit $f$-divergence constraints.} For any convex function $f$, the corresponding $f$-divergence is defined as:
\begin{equation}
\label{eqn:primal_f}
    D_f(\policy(\cdot|\bs), \behavior(\cdot|\bs)) = \E_{\ba \sim \policy(\cdot|\bs)} \left[ f \left( \frac{\policy(\ba|\bs)}{\behavior(\ba|\bs)} \right) \right].
\end{equation}
KL-divergence, $\chi^2$-divergence, and total-variation distance are some commonly used members of the $f$-divergence family, corresponding to different choices of function $f$~\citep{nowozin2016f}. A variational form for the $f$-divergence can also be written as
\begin{equation}
\label{eqn:variational}
    D_f(\policy(\cdot|\bs), \behavior(\cdot|\bs)) = \max_{x: S \times A \rightarrow \mathbb{R}} \E_{\ba \sim \policy(\cdot|\bs)}\left[ x(\bs, \ba) \right] - \E_{\ba' \sim \behavior(\cdot|\bs)}\left[ f^*(x(\bs, \ba')) \right],
\end{equation}
where $f^*$ is the convex conjugate for the convex function $f$. Both the primal form (Equation~\ref{eqn:primal_f}) and the dual variational form (Equation~\ref{eqn:variational}) of the $f$-divergence has been used to specify policy constraints. In the dual form, an additional neural network is used to represent the function $x$.
The standard form of ``passive dynamics'' in linearly solvable MDPs~\citep{todorov_lmdp} or the action prior in control as inference~\citep{levine2018reinforcement} corresponds to the KL-divergence (primal form), which has also been used in recent work that adapts such a KL-divergence penalty for offline reinforcement learning~\citep{jaques2019way,wu2019behavior}. The KL-divergence, given by $D_{\mathrm{KL}}(\policy, \behavior) = \E_{\ba \sim \policy(\cdot|\bs)}[\log \policy(\ba|\bs) - \log \behavior(\ba|\bs)]$, can be computed by sampling action samples $\ba \sim \policy(\cdot|\bs)$, and then computing sample-wise estimates of the likelihoods inside the expectation. It is commonly used together with ``policy penalty'' algorithms, by simply adding an estimate of $-\alpha \log \behavior(\ba|\bs)$ to the reward function, and employing an entropy regularized reinforcement learning algorithm, which analytically adds $\ent(\policy(\cdot|\bs))$ to the actor objective~\citep{levine2018reinforcement}. One significant disadvantage of this approach is that it requires explicit estimation of the behavior policy (e.g., via behavioral cloning) to fit $\log \behavior(\ba|\bs)$.

Additionally, the sub-family of asymmetrically-relaxed $f$-divergences can be used for the policy constraint. For any chosen convex function $f$, these divergences modify the expression for $D_f$ to integrate over only those $\ba$ such that the density ration $\policy(\cdot|\bs)/\behavior(\ba|\bs)$ is larger than a pre-defined threshold, thereby not penalizing small density ratios. \citet{wu2019behavior} briefly discuss this divergence sub-family, and we refer readers to \citet{wu2019domain} for a detailed description.

\paragraph{Implicit $f$-divergence constraints.} The KL-divergence constraint can also be enforced implicitly, as in the case of AWR~\citep{peng2019awr}, AWAC~\citep{nair2020accelerating}, and ABM~\citep{siegel2020keep}. These methods first solve for the optimal next policy iterate under the KL-divergence constraint, indicated as $\bar{\policy}_{k+1}$, non-parameterically, and then project it onto the policy function class via supervised regression, implementing the following procedure:
\begin{align*}
\bar{\policy}_{k+1}(\ba|\bs) &\leftarrow \frac{1}{Z}\behavior(\ba|\bs)\exp\left( \frac{1}{\alpha} Q^\policy_k(\bs,\ba) \right) \\
\policy_{k+1} &\leftarrow \arg\min_\policy D_{\mathrm{KL}}(\bar{\policy}_{k+1}, \policy)
\end{align*}
In practice, the first step can be implemented by weighting samples from $\behavior(\ba|\bs)$ (i.e., the data in the buffer $\data$) by importance weights proportional to $\exp\left( \frac{1}{\alpha} Q^\policy_k(\bs,\ba) \right)$, and the second step can be implemented via a weighted supervised learning procedure employing these weights. In this way, importance sampling effectively implements a KL-divergence constraint on the policy, with $\alpha$ corresponding to the Lagrange multiplier. The Q-function $Q^\policy_k$ corresponding to the previous policy $\policy_k$ can be estimated with any Q-value or advantage estimator. We refer the reader to related work for a full derivation of this procedure~\citep{peng2019awr,nair2020accelerating}.

\paragraph{Integral probability metrics (IPMs).} Another choice of the policy constraint $D$ is an integral probability metric, which is a divergence measure with the following functional form dependent on a function class $\Phi$:
\begin{equation}
    D_\Phi(\policy(\cdot|\bs), \behavior(\cdot|\bs)) = \sup_{\phi \in \Phi, \phi: S \times A \rightarrow \mathbb{R}} \left\vert \E_{\ba \sim \policy(\cdot|\bs)}[\phi(\bs, \ba)] - \E_{\ba' \sim \behavior(\cdot|\bs)}[\phi(\bs, \ba')] \right\vert.
\end{equation}
Different choices for the function class $\Phi$ give rise to different divergences. For example, when $\Phi$ consists of all functions with a unit Hilbert norm in a reproducing kernel Hilbert space (RKHS), $D_\Phi$ represents the maximum mean discrepancy (MMD) distance, which can alternatively be computed directly from samples as following:
\begin{multline*}
    \text{MMD}^2(\policy(\cdot|\bs), \behavior(\cdot|\bs)) = \E_{\ba \sim \policy(\cdot|\bs), \ba' \sim \policy(\cdot|\bs)}\left[ k(\ba, \ba') \right] -\\ 2 \E_{\ba \sim \policy(\cdot|\bs), \ba' \sim \behavior(\cdot|\bs)}\left[ k(\ba, \ba') \right] + \E_{\ba \sim \behavior(\cdot|\bs), \ba' \sim \behavior(\cdot|\bs)}\left[ k(\ba, \ba') \right],
\end{multline*}
where $k(\cdot, \cdot)$ represents any radial basis kernel, such as the Gaussian or Laplacian kernel. As another example, when $\Phi$ consists of all functions $\phi$ with a unit Lipschitz constant, then $D_\Phi$ is equivalent to the first order Wasserstein ($W_1$) or Earth-mover's distance:
\begin{equation}
    W_1(\policy(\cdot|\bs), \behavior(\cdot|\bs)) = \sup_{f, ||f||_L \leq 1} \left\vert \E_{\ba \sim \policy(\cdot|\bs)}[f(\ba)] - \E_{\ba \sim \behavior(\cdot|\bs}[f(\ba)] \right\vert.
\end{equation}
These metrics are appealing because they can often be estimated with non-parametric estimators. We refer the readers to \citet{sriperumbudur2009integral} for a detailed discussion on IPMs. BEAR~\citep{kumar2019stabilizing} uses the MMD distance to represent the policy constraint, while \citet{wu2019behavior} evaluates an instantiation of the first order Wasserstein distance.

\paragraph{Tradeoffs between different constraint types.}
The KL-divergence constraint, as well as other $f$-divergences, represent a particularly convenient class of constraints, since they readily lend themselves to be used in a policy penalty algorithm by simply augmenting the reward function according to $\bar{\reward}(\bs,\ba) = \reward(\bs,\ba) - \alpha f(\policy(\ba|\bs)/\behavior(\ba|\bs))$, with the important special case of the KL-divergence corresponding to $\bar{\reward}(\bs,\ba) = \reward(\bs,\ba) + \alpha\log \behavior(\ba|\bs) - \alpha \log \policy(\ba|\bs)$, with the last term subsumed inside the entropy regularizer $\ent(\policy(\cdot|\bs))$ when using a maximum entropy reinforcement learning algorithm~\citep{levine2018reinforcement}. However, KL-divergences and other $f$-divergences are not necessarily ideal for offline reinforcement learning. Consider a setting where the behavior policy is uniformly random. In this case, offline reinforcement learning should, in principle, be very effective. In fact, even standard actor-critic algorithms \emph{without} any policy constraints can perform very well in this setting, since when all actions have equal probability, there are no out-of-distribution actions~\citep{kumar2019stabilizing}. However, a KL-divergence constraint in this setting would restrict the learned policy $\policy(\ba|\bs)$ from being too concentrated, requiring the constrained algorithm to produce a highly stochastic (and therefore highly suboptimal) policy. Intuitively, an effective policy constraint should prevent the learned policy $\policy(\ba|\bs)$ from going \emph{outside} the set of actions that have a high probability in the data, but would not prevent it from concentrating around a \emph{subset} of high-probability actions. This suggests that a KL-divergence constraint is not in general the ideal choice.

\begin{wrapfigure}{r}{0.6\textwidth} 
    \centering
    \vspace{-5pt}
    \includegraphics[width=1.0\linewidth]{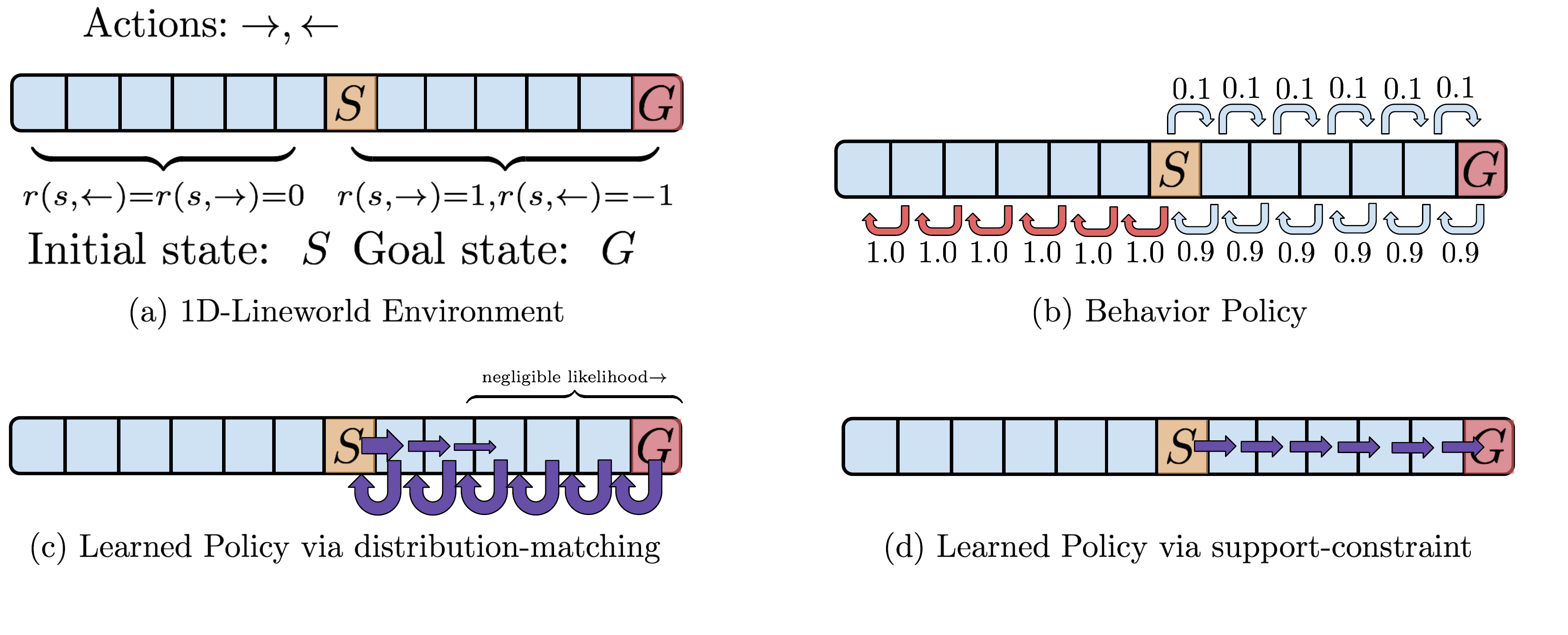}
    \vspace{-15pt}
    \caption{A comparison of support and distribution constraints on a simple 1D lineworld from \citet{kumar_blog}. The task requires moving to the goal location (marked as 'G') starting from 'S'. The behavior policy strongly prefers the left action at each state (b), such that distribution constraint is unable to find the optimal policy (c), whereas support-constraint can successfully obtain the optimal policy (d). We refer to \citet{kumar_blog} for further discussion. \label{fig:support_vs_distribution}}
    \vspace{-10pt}
\end{wrapfigure}
In contrast, as argued by \citet{kumar2019stabilizing} and \citet{laroche2017safe}, restricting the \textit{support} of the learned policy $\policy(\ba|\bs)$ to the support of the behavior distribution $\behavior(\ba|\bs)$ may be sufficient to theoretically and empirically obtain an effective offline RL method. In the previous example, if only the support of the learned policy is constrained to lie in the support of the behavior policy, the learned policy can be deterministic and optimal. However, when the behavior policy does not have full support, a support constraint will still prevent out-of-distribution actions. \citet{kumar_blog} formalize this intuition and present a simple example of a 1D-lineworld environment, which we reproduce in Figure~\ref{fig:support_vs_distribution} where constraining distributions can lead to highly suboptimal behavior, that strongly prefers moving leftwards, thus reaching the goal location with only very low likelihood over the course of multiple intermediate steps of decision making, while support constraints still allow for learning the optimal policy, since they are agnostic to the probability density function in this setting.

Which divergence metrics can be used to constrain policy supports? In a finite sample setting, the family of $f$-divergences measures the difference in the \textit{probability densities}, which makes it unsuitable for support matching. In an MDP with a discrete action-space, a simple choice for constraining the support is a metric that penalizes the total amount of probability mass on out-of-distribution actions under the $\policy$ distribution, as shown below:
\begin{equation}
    D_{\text{support}, \varepsilon}(\policy(\cdot|s), \behavior(\cdot|s)) = \sum_{\ba \in A, \behavior(\ba|\bs) \leq \varepsilon} \policy(\ba|\bs).
\end{equation}
This metric has been used in the context of off-policy bandits~\citep{sachdeva2020offpolicy}, but not in the context of offline reinforcement learning.
When the underlying MDP has a continuous action space and exact support cannot be estimated, \citet{kumar2019stabilizing} show that a finite sample estimate of the MMD distance can be used to approximately constrain supports of the learned policy. Similar to $f$-divergences, the MMD distance still converges to a divergence estimate between the distribution functions of its arguments. However, \citet{kumar2019stabilizing} show experimentally (Figure 7) that, in the finite sample setting, MMD resembles a support constraining metric. The intuition is that the MMD distance does not depend on the specific densities of the behavior distribution or the policy, and can be computed via a kernel-based distance on samples from each distribution, thus making it just sufficient enough for constraining supports when finite samples are used. Some variants of the asymmetric $f$-divergence that asymmetrically penalize the density ratios $\policy(\cdot|\bs)/\behavior(\cdot|\bs)$ can also be used to approximately constrain supports~\citep{wu2019domain,wu2019behavior}.

\subsection{Offline Approximate Dynamic Programming with Uncertainty Estimation}
\label{sec:adp_uncertainty}

As discussed in Section~\ref{sec:policy_constraints}, policy constraints can prevent out-of-distribution action queries to the Q-function when computing the target values. Aside from directly constraining the policy, we can also mitigate the effect of out-of-distribution actions by making the Q-function resilient to such queries, via effective uncertainty estimation. The intuition behind these \textbf{uncertainty-based} methods is that, if we can estimate the \emph{epistemic} uncertainty of the Q-function, we expect this uncertainty to be substantially larger for out-of-distribution actions, and can therefore utilize these uncertainty estimates to produce conservative target values in such cases. In this section, we briefly review such uncertainty-aware formulations of offline approximate dynamic programming methods.

Formally, such methods require learning an uncertainty set or distribution over possible Q-functions from the dataset $\data$, which we can denote $\mathcal{P}_\data(Q^{\pi})$. This can utilize explicit modeling of confidence sets, such as by maintaining upper and lower confidence bounds~\citep{jaksch2010near}, or directly representing samples from the distribution over Q-functions, for example via bootstrap ensembles~\citep{osband2016deep}, or by parameterizing the distribution over Q-values using a known (e.g., Gaussian) distribution with learned parameters~\citep{o2018uncertainty}. If calibrated uncertainty sets can be learned, then we can improve the policy using a conservative estimate of the Q-function, which corresponds to the following policy improvement objective:
\begin{equation}
    \pi_{k+1} \leftarrow \arg\max_\policy \E_{\bs \sim \data}\underbrace{\left[
\E_{\ba \sim \policy(\ba|\bs)}\left[\E_{Q^\policy_{k+1} \sim \mathcal{P}_\data(Q^{\pi})} [Q^\policy_{k+1}(\bs,\ba)] -  \alpha \text{Unc}(\mathcal{P}_\data(Q^{\pi}))\right]\right]}_{\text{conservative estimate}}, \label{eq:uncertainty_conservative}
\end{equation}
where $\text{Unc}$ denotes a metric of uncertainty, such that subtracting it provides a conservative estimate of the \textit{actual} Q-function. The choice of the uncertainty metric $\text{Unc}$ depends on the particular choice of uncertainty estimation method, and we discuss this next.

When bootstrap ensembles\footnote{It is known in the deep learning literature that, although a true bootstrap ensemble requires resampling the dataset with replacement for every bootstrap, omitting this resampling step and simply using different random initialization for every neural network in the ensemble is sufficient to differentiate the models and provide reasonable uncertainty estimates~\citep{osband2016deep}. In fact, previous work has generally observed little benefit from employ proper resampling~\citep{osband2016deep}, which technically means that all of these methods simply employ regular (non-bootstrapped) ensembles, although the quality of their uncertainty estimates is empirically similar.} are used to represent the Q-function, as is common in prior work~\citep{osband2016deep,eysenbach2017leave,kumar2019stabilizing,agarwal2019optimistic}, typical choices of `$\text{Unc}$' include variance across Q-value predictions of an ensemble of Q-functions~\citep{kumar2019stabilizing}, or maximizing the Q-value with respect to the worst case or all convex combinations of the Q-value predictions of an ensemble~\citep{agarwal2019optimistic}. When a parametric distribution, such as a Gaussian, is used to represent the Q-function~\citep{o2018uncertainty}, a choice of $\text{Unc}$, previously used for exploration, is to use a standard deviation above the mean as an optimistic Q-value estimate for policy improvement. When translated to offline RL, an analogous estimate would be to use conservative Q-values, such as one standard deviation below the mean, for policy improvement.

\subsection{Conservative Q-Learning and Pessimistic Value Functions}
\label{sec:adp_conservative}

As an alternative to imposing constraints on the policy in an actor-critic framework, an effective approach to offline RL can also be developed by regularizing the value function or Q-function directly to avoid overestimation for out-of-distribution actions~\citep{kumar2020conservative}. This approach can be appealing for several reasons, such as being applicable to both actor-critic and Q-learning methods, even when a policy is not represented explicitly, and avoiding the need for explicit modeling of the behavior policy. The practical effect of such a method resembles the conservative estimate in Equation~(\ref{eq:uncertainty_conservative}), but without explicit uncertainty estimation. As discussed by \citet{kumar2020conservative}, one simple way to ensure a conservative Q-function is to modify the objective for fitting the Q-function parameters $\phi$ on Line 14 of the Q-learning method in Algorithm~\ref{alg:qlearning} or Line 14 of the actor-critic method in Algorithm~\ref{alg:actorcritic} to add an additional \emph{conservative penalty} term, yielding a modified objective
\[
\tilde{\en}(B,\phi) = \alpha \conpen(B,\phi) + \en(B,\phi),
\]
where different choices for $\conpen(B,\phi)$ lead to algorithms with different properties. A basic example is the penalty $\conpen_{\text{CQL}_0}(B,\phi) = \E_{\bs \sim B, \ba \sim \mu(\ba|\bs)}[Q_\phi(\bs,\ba)]$. Intuitively, this penalty minimizes the Q-values at all of the states in the buffer, for actions selected according to the distribution $\mu(\ba|\bs)$. If $\mu(\ba|\bs)$ is chosen \emph{adversarially}, for example by \emph{maximizing} the penalty $\conpen_{\text{CQL}_0}(B,\phi)$, the effect is that the conservative penalty will push down on high Q-values. Note that the standard Bellman error term $\en(B,\phi)$ will still enforce that the Q-values obey the Bellman backup for \emph{in-distribution} actions. Therefore, if the penalty weight $\alpha$ is chosen appropriately, the conservative penalty should mostly push down on Q-values for out-of-distribution actions for which the Q-values are (potentially erroneously) high, since in-distribution actions would be ``anchored'' by the Bellman error $\en(B,\phi)$. Indeed, \citet{kumar2020conservative} show that, for an appropriately chosen value of $\alpha$, a Q-function trained with this conservative penalty will represent a \emph{lower bound} on the true Q-function $Q(\bs,\ba)$ for the current policy, both in theory and in practice. This results in a provably conservative Q-learning or actor-critic algorithm.
A simple and practical choice for $\mu(\ba,\bs)$ is to use a regularized adversarial objective, such that
\[
\mu = \arg\max_\mu \E_{\bs \sim \data} \left[ \E_{\ba \sim \mu(\ba|\bs)}[Q_\phi(\bs,\ba)] + \ent(\mu(\cdot | \bs))\right].
\]
This choice inherits the lower bound guarantee above, and is simple to compute. The solution to the above optimization is given by $\mu(\ba|\bs) \propto \exp(Q(\bs,\ba))$, and we can express $\conpen_{\text{CQL}_0}(B,\phi)$ under this choice of $\mu(\ba|\bs)$ in closed form as \mbox{$\conpen_{\text{CQL}_0}(B,\phi) = E_{\bs \sim B}[\log \sum_\ba \exp(Q_\phi(\bs,\ba))]$}. This expression has a simple intuitive interpretation: the log-sum-exp is dominated by the action with the largest Q-value, and hence this type of penalty tends to minimize whichever Q-value is largest at each state. When the action space is large or continuous, we can estimate this quantity by sampling and reweighting. For example, \citet{kumar2020conservative} propose sampling from the current actor (in an actor-critic algorithm) and computing importance weights.

While the approach described above has the appealing property of providing a learned Q-function $Q_\phi(\bs,\ba)$ that lower bounds the true Q-function, and therefore provably avoids overestimation, it tends to suffer from excessive \emph{underestimation} -- that is, it is too conservative. A simple modification, which we refer to as $\conpen_{\text{CQL}_1}(B,\phi)$, is to also add a value \emph{maximization} term to balance out the minimization term under $\mu(\ba|\bs)$, yielding the following expression:
\[
\conpen_{\text{CQL}_1}(B,\phi) = \E_{\bs \sim B, \ba \sim \mu(\ba|\bs)}[Q_\phi(\bs,\ba)] - \E_{(\bs,\ba) \sim B}[Q_\phi(\bs,\ba)].
\]
This conservative penalty minimizes Q-values under the adversarially chosen $\mu(\ba|\bs)$ distribution, and \emph{maximizes} the values for state-action tuples in the batch. Intuitively, this acts to ensure that high Q-values are only assigned to in-distribution actions. When $\mu(\ba|\bs)$ is equal to the behavior policy, the penalty is zero on average. While this penalty does \emph{not} produce a Q-function that is a pointwise lower bound on the true Q-function, it is a lower bound in expectation under the current policy, thereby still providing appealing conservatism guarantees, while substantially reducing underestimation in practice. As shown by \citet{kumar2020conservative}, this variant produces the best performance in practice.

\subsection{Challenges and Open Problems}
\label{sec:adp_challenges}

As we discussed in Section~\ref{sec:basic_offline_adp}, basic approximate dynamic programming algorithms can perform very poorly in the offline setting due to distributional shift, primarily due to the distributional shift of the actions due to the discrepancy between the behavior policy $\behavior(\ba|\bs)$ and the current learned policy $\policy(\ba|\bs)$, which is used in the target value calculation for the Bellman backup. We discussed how policy constraints and explicit uncertainty estimation can in principle mitigate this problem, but both approaches have a number of shortcomings.

While such explicit uncertainty-based methods are conceptually attractive, it is often very hard to obtain calibrated uncertainty estimates to evaluate $\mathcal{P}_\data(\hat{Q}^\pi)$ and $\text{Unc}$ in practice, especially with modern high-capacity function approximators, such as deep neural networks.
In practice, policy constraint and conservative value function methods so far seem to outperform pure uncertainty-based methods~\citep{fujimoto2018off}. This might at first appear surprising, since uncertainty estimation has been a very widely used tool in another subfield of reinforcement learning -- exploration. In exploration, acting \emph{optimistically} with respect to estimated uncertainty, or utilizing posterior sampling, has been shown to be effective in practice~\citep{osband2017posterior}. However, in the setting of offline reinforcement learning, where we instead must act \emph{conservatively} with respect to the uncertainty set, the demands on the fidelity of the uncertainty estimates are much higher. Exploration algorithms only require the uncertainty set to include good behavior -- in addition, potentially, to a lot of bad behavior. However, offline reinforcement learning requires the uncertainty set to directly capture the trustworthiness of the Q-function, which is a much higher bar. In practice, imperfect uncertainty sets can give rise to either overly conservative estimates, which impedes learning, or overly loose estimates, which results in exploitation of out-of-distribution actions. Of course, the relatively performance of policy constraint and uncertainty-based methods may change in the future, as the community develops better epistemic uncertainty estimation techniques or better algorithms to incorporate more suitable distribution metrics into offline RL.

Policy constraint methods do however suffer from a number of challenges. First, most of these methods fit an estimated model of the behavior policy $\behavior(\cdot|\bs)$ from the dataset and constrain the learned policy $\policy(\cdot|\bs)$ against this estimated behavior policy. This means that the performance of these algorithms is limited by the accuracy of estimation of the behavior policy, which may be hard in scenarios with highly multimodal behaviors, as is the case in practice with real-world problems. For example, if a unimodal Gaussian policy is used to model a highly multi-modal action distribution, averaging across different modes while fitting this behavior policy may actually be unable to prevent the learned policy, $\policy$, from choosing out-of-distribution actions. Methods that enforce the constraint \emph{implicitly}, using only samples and without explicit behavior policy estimation, are a promising way to alleviate this limitation~\citep{peng2019awr,nair2020accelerating}.

Even when the behavior policy can be estimated exactly, a number of these methods still suffer from the undesirable effects of function approximation. When neural networks are used to represent Q-functions, undesirable effects of function approximation may prevent the Q-function from learning meaningful values even when out-of-distribution actions are controlled for in the target values. For example, when the size of the dataset is limited, approximate dynamic programming algorithms tend to overfit on the small dataset, and this error is then accumulated via Bellman backups~\citep{fu2019diagnosing,kumar2020discor}. Moreover, if the dataset state-action distribution is narrow, neural network training may only provide brittle, non-generalizable solutions. Unlike online reinforcement learning, where accidental overestimation errors arising due to function approximation can be corrected via active data collection, these errors cumulatively build up and affect future iterates in an offline RL setting.

Methods that estimate a conservative or pessimistic value function~\citep{kumar2020conservative} present a somewhat different set of tradeoffs. While they avoid issues associated with estimating the behavior policy and can effectively avoid overestimation, they can instead suffer from underestimation and a form of overfitting: if the dataset is small, the conservative regularizer can assign values that are too low to actions that are undersampled in the dataset. Indeed, \emph{excessive} pessimism may be one of the bigger issues preventing better performance on current benchmarks, and an important open question is how to dynamically modulate the degree of conservatism to balance the risks of overestimation against the downside of avoiding any unfamiliar action.

A further issue that afflicts both constraint-based and conservative methods is that, while the Q-function is never evaluated on out-of-distribution states during training, the Q-function is still affected by the training set state distribution $\freq^{\behavior}(\bs)$, and this is not accounted for in current offline learning methods. For instance, when function approximation couples the Q-value at two distinct states with very different densities under $\freq^{\behavior}(\bs)$, dynamic programming with function approximation may give rise to incorrect solutions on the state that has a lower probability $\freq^{\behavior}(\bs)$. A variant of this issue was noted in the standard RL setting, referred to as an absence of ``corrective feedback'' by \citet{kumar2020discor} (see \citet{discor_blog} for a short summary), and such a problem may affect offline RL algorithms with function approximation more severely, since they have no control over the data distribution at all.

Another potential challenge for all of these offline approximate dynamic programming methods is that the degree of improvement beyond the behavior policy is restricted by error accumulation. Since the error in policy performance compounds with a factor that depends on $1 / (1 - \gamma)^2 \approx H^2$~\citep{farahmand2010error,kumar2019stabilizing,kidambi2020morel}, a small divergence from the behavior policy at each step can give rise to a policy that diverges away from the behavior distribution and performs very poorly. Besides impacting training, this issue can also lead to severe \emph{state} distribution shift at test time, which can lead the policy to perform very poorly. Therefore, policy constraints must be strong, so as to ensure that the overall error is small. This can restrict the amount of policy improvement that can be achieved. We might expect that directly constraining the state-action marginal distribution of the policy, such as the methods explored in recent work~\citep{nachum2019algaedice} might not suffer from the error accumulation issue, however, \citet{kidambi2020morel} proved a lower-bound result showing that quadratic scaling with respect to horizon is inevitable in the worst case for any offline RL method. 
Moreover, as previously discussed in Section~\ref{sec:pg_challenges}, representing and enforcing constraints on the state-action marginal distributions themselves requires Bellman backups, which can themselves suffer from out-of-distribution actions.
Besides the worst-case dependence on the horizon, an open question that still remains is is the development of constraints that can effectively trade off error accumulation and suboptimality of the learned policy in most ``average''-case MDP instances, and can be easily enforced and optimized in practice via standard optimization techniques without requiring additional function approximators to fit the behavior policy.

\section{Offline Model-Based Reinforcement Learning}
\label{sec:model_based}

The use of predictive models can be a powerful tool for enabling effective offline reinforcement learning. Since model-based reinforcement learning algorithms primarily rely on their ability to estimate $\transitions(\bs_{t+1}|\bs_t,\ba_t)$ via a parameterized model $\transitions_\psi(\bs_{t+1}|\bs_t,\ba_t)$, rather than performing dynamic programming or importance sampling, they benefit from convenient and powerful supervised learning methods when fitting the model, allowing them to effectively utilize large and diverse datasets. However, like dynamic programming methods, model-based offline RL methods are not immune to the effects of distribution shift. In this section, we briefly discuss how distributional shift affects model-based reinforcement learning methods, then survey a number of works that utilize models for offline reinforcement learning, and conclude with a brief discussion of open challenges. A complete treatment of all model-based reinforcement learning work, as well as work that learns predictive models of physics but does not utilize them for control (e.g., \citet{lerer2016learning,battaglia2016interaction}), is outside the scope of this tutorial, and we focus primarily on work that performs both model-fitting and control from offline data.

\subsection{Model Exploitation and Distribution Shift}

As discussed in Section~\ref{sec:rl_prelims}, the model in a model-based RL algorithm can be utilized either for planning (including online, via MPC) or for training a policy. In both cases, the model must provide accurate predictions for state-action tuples that are more optimal with respect to the current task. That is, the model will be queried at $\bs \sim \freq^\policy(\bs)$, where $\policy$ is either an explicit policy, or an implicit policy produced by running planning under the model. In general $\freq^\policy(\bs) \neq \freq^\behavior(\bs)$, which means that the model is itself susceptible to distribution shift. In fact, the model suffers from distribution shift both in terms of the state distribution $\freq^\policy(\bs)$, and the action distribution $\policy(\ba|\bs)$.

Since the policy (or action sequence, in the case of planning) is optimized to obtain the highest possible expected reward under the current model, this optimization process can lead to the policy \emph{exploiting} the model, intentionally producing out-of-distribution states and actions at which the model $\transitions_\psi(\bs_{t+1}|\bs_t,\ba_t)$ erroneously predicts successor states $\bs_{t+1}$ that lead to higher returns than the actual successor states that would be obtained in the real MDP. This \emph{model exploitation} problem can lead to policies that produce substantially worse performance in the real MDP than what was predicted under the model. Prior work in \emph{online} model-based RL has sought to address this issue primarily via uncertainty estimation, analogously to the uncertainty-based methods discussed in Section~\ref{sec:adp_uncertainty}, but this time modeling the epistemic uncertainty of the learned model $\transitions_\psi(\bs_{t+1}|\bs_t,\ba_t)$. In low-dimensional MDPs, such uncertainty estimates can be produced by means of Bayesian models such as Gaussian processes~\citep{deisenroth2011pilco}, while for higher-dimensional problems, Bayesian neural networks and bootstrap ensembles can be utilized~\citep{pets,mbpo}. Effective uncertainty estimation is generally considered an important component of modern model-based reinforcement learning methods, for the purpose of mitigating model exploitation.

Theoretical analysis of model-based policy learning can provide bounds on the error incurred from the distributional shift due to the divergence between the learned policy $\policy(\ba|\bs)$ and the behavior policy $\behavior(\ba|\bs)$~\citep{sun2018dual,luo2018algorithmic,mbpo}. This analysis resembles the distributional shift analysis provided in the DAgger example in Section~\ref{sec:basic_challenges}, except that now both the policy and the transition probabilities experience distributional shift. Following \citet{mbpo}, if we assume that the total variation distance (TVD) between the learned model $\transitions_\psi$ and true model $\transitions$ is bounded by $\epsilon_m = \max_t \E_{\freq_t^\policy} D_{TV}(\transitions_\psi(s_{t+1}|s_t,a_t) \| \transitions(s_{t+1}|s_t,a_t))$, and the TVD between $\policy$ and $\behavior$ is likewise bounded on sampled states by $\epsilon_\policy$, then the true policy value $J(\policy)$ is related to the policy estimate under the model, $J_\psi(\policy)$, according to
\[
J(\policy) \geq J_\psi(\policy) - \left[
\frac{2\discount\reward_{\text{max}} (\epsilon_m + 2\epsilon_\policy)}{(1-\discount)^2} +
\frac{4\reward_{\text{max}} \epsilon_\policy}{1-\discount}
\right].
\]
Intuitively, the second term represents error accumulation due to the distribution shift in the policy, while the first term represents error accumulation due to the distribution shift in the model. The first term also includes a dependence on $\epsilon_\policy$, because policies that diverge more from $\behavior$ will lead to states that are further outside of the data distribution, which in turn will make errors in the model more likely. \citet{mbpo} also argue that a modified model-based RL procedure that resembles Dyna~\cite{sutton1991dyna}, where only short-horizon rollouts from the model are generated by ``branching'' off of states seen in the data, can mitigate this accumulation of error. This is also intuitively natural: if applying the learned model recursively leads to compounding error, then shorter rollouts should incur substantially less error.

\subsection{Brief Survey of Model-Based Offline Reinforcement Learning}

A natural and straightforward way to utilize model-based reinforcement learning algorithms in the offline setting is to simply train the model from offline data, with minimal modification to the algorithm. As with Q-learning and actor-critic methods, model-based reinforcement learning algorithms can be applied to the offline setting na\"{i}vely. Furthermore, many effective model-based reinforcement learning methods \emph{already} take steps to limit model exploitation via a variety of uncertainty estimation methods~\citep{deisenroth2011pilco,pets}, making them reasonably effective in the offline setting. Many such methods have been known to exhibit excellent performance in conventional off-policy settings, where additional data collection is allowed, but prior data is also utilized~\citep{sutton1991dyna,watter2015embed,zhang2018solar,hafner2018learning,mbpo}. Indeed, \citet{yu2020mopo} showed that MBPO~\citep{mbpo} actually attains reasonable performance on standard offline RL settings without modification, whereas na\"{i}ve dynamic programming methods (e.g., soft actor-critic~\citep{sac}) fail to learn meaningful policies without policy constraints. This suggests that model-based RL algorithms are likely to lead to an effective class of offline reinforcement learning methods.

\paragraph{Offline RL with standard model-based RL methods.} A number of recent works have also demonstrated effective offline learning of predictive models and their application to control in complex and high-dimensional domains, including settings with image observations. Some of these methods have directly utilized high-capacity predictive models on high-dimensional observations, such as images, directly employing for online trajectory optimization (i.e., MPC). Action-conditional convolutional neural networks~\citep{oh2015action} have been used to provide accurate, long-term prediction of behavior in Atari games and have been combined with RL to improve sample-efficiency over model-free methods~\citep{kaiser2019model}. The \emph{visual foresight} method involves training a video prediction model to predict future image observations for a robot, conditioned on the current image and future sequence of actions. The model is represented with a recurrent neural network, and trained on large amounts of offline data, typically collected with an uninformative randomized policy~\citep{finn2017deep,ebert2018visual}. More recent work has demonstrated this approach on very large and diverse datasets, collected from multiple viewpoints, many objects, and multiple robots, suggesting a high degree of generalization, though the particular behaviors are comparatively simple, typically relocating individual objects by pushing or grasping~\cite{dasari2019robonet}. A number of prior works have also employed ``hybrid'' methods that combine elements of model-free and model-based algorithms, predicting future rewards or reward features conditioned on a sequence of future actions, but avoiding direct prediction of future observations. A number of such methods have been explored in the conventional online RL setting~\citep{tamar2016value,dosovitskiy2016learning,oh2017value}, and in the offline RL setting, such techniques have been applied effectively to learning navigational policies for mobile robots from previously collected data~\citep{kahn2018composable,kahn2020badgr}.

\paragraph{Off-policy evaluation with models.} Model learning has also been explored considerably in the domain of off-policy evaluation (OPE), as a way to reduce the variance of importance sampled estimators of the sort discussed in Section~\ref{sec:ope}. Here, the model is used to provide a sort of baseline for the expected return inside of an importance sampled estimator, as illustrated in Equation~(\ref{eq:doubly_robust}) in Section~\ref{sec:ope}. In these settings, the model is typically trained from offline data~\citep{jiang2015doubly,thomas2016data,farajtabar2018more,wang2017optimal}.

\paragraph{Distribution and safe region constraints.} As with the policy constraint methods in Section~\ref{sec:policy_constraints}, model-based RL algorithms can also utilize constraints imposed on the policy or planner with the learned model. Methods that guarantee Lyapunov stability~\citep{berkenkamp2017safe} of the learned model can be used to constrain policies within ``safe'' regions of the state space where the chance of failure is small. Another example of such an approach is provided by deep imitative models (DIMs)~\citep{rhinehart2018deep}, which learn a normalizing flow model over future trajectories from offline data, conditioned on a high-dimensional observation. Like the hybrid methods described above, DIMs avoid direct prediction of observations. The learned distribution over dataset trajectories can then be used to both provide predictions for a planner, and provide a distributional constraint, preventing the planner from planning trajectories that deviate significantly from the training data, thus limiting distributional shift. 

\paragraph{Conservative model-based RL algorithms.} More recently, two concurrent methods, MoREL~\citep{kidambi2020morel} and MOPO~\citep{yu2020mopo}, have proposed offline model-based RL algorithms that aim to utilize conservative value estimates to provide analytic bounds on performance. Conceptually, these methods resemble the conservative estimation approach described in Section~\ref{sec:adp_conservative}, but instead of regularizing a value function, they modify the MDP model learned from data to induce conservative behavior. The basic principle is to provide the policy with a penalty for visiting states under the model where the model is likely to be incorrect. If the learned policy takes actions that remain in regions where the model is accurate, then the model-based estimate of the policy's value (and its gradient) will likely be accurate also. Let $u(\bs,\ba)$ denote an \emph{error oracle} that provides a consistent estimate of the accuracy of the model at state-action tuple $(\bs,\ba)$. For example, as proposed by \citet{yu2020mopo}, this oracle might satisfy the property that $D(\transitions_\psi(s_{t+1}|s_t,a_t) \| \transitions(s_{t+1}|s_t,a_t)) \leq u(\bs,\ba)$ for some divergence measure $D(\cdot,\cdot)$. Then, conservative behavior can be induced either by modifying the reward function to obtain a conservative reward of the form $\tilde{r}(\bs,\ba) = r(\bs,\ba) - \lambda u(\bs,\ba)$, as in MOPO~\citep{yu2020mopo}, or by modifying the MDP under the learned model so that the agent enters an absorbing state with a low reward value when $u(\bs,\ba)$ is below some threshold, as in MoREL~\citep{kidambi2020morel}. In both cases, it is possible to show that the model-based estimate of the policy's performance under the model the modified reward function or MDP structure bounds that policy's true performance in the real MDP, providing appealing theoretical guarantees on performance. Note, however, that such approaches still require a consistent estimator for the error oracle $u(\bs,\ba)$. Prior work has used disagreement in a bootstrap ensemble to provide this estimate, but it is not guaranteed to be consistent under sampling error, and resolving this limitation is likely to be an important direction for future work.

\subsection{Challenges and Open Problems}

Although model-based reinforcement learning appears to be a natural fit for the offline RL problem setting, current methods for fully offline model-based reinforcement learning generally rely on explicit uncertainty estimation for the model to detect and quantify distributional shift, for example by using a bootstrap ensemble. Although uncertainty estimation for models is in some ways more straightforward than uncertainty estimation for value functions, current uncertainty estimation methods still leave much to be desired, and it seems likely that offline performance of model-based RL methods can be improved substantially by carefully accounting for distributional shift.

Model-based RL methods also present their own set of challenges: while some MDPs are easy to model accurately, others can be exceedingly difficult. Modeling MDPs with very high-dimensional image observations and long horizons is a major open problem, and current predictive modeling methods generally struggle with long-horizon prediction. Hybrid methods that combine model-based and model-free learning, for example by utilizing short rollouts~\citep{sutton1991dyna,mbpo} or avoiding prediction of full observations~\citep{dosovitskiy2016learning,oh2017value,kahn2020badgr} offer some promise in this area.

It is also still an open theoretical question as to whether model-based RL methods even \emph{in theory} can improve over model-free dynamic programming algorithms. The reasoning behind this question is that, although dynamic programming methods do not \emph{explicitly} learn a model, they essentially utilize the dataset as a ``non-parametric'' model. Fundamentally, both dynamic programming methods and model-based RL methods are solving \emph{prediction} problems, with the former predicting future returns, and the latter predicting future states. Moreover, model-free methods can be modified to predict even more general quantities~\citep{sutton2011horde}, such as return values with different discount factors, return values for different number of steps into the future, etc. In the linear function approximation case, it is known that model-based updates and fitted value iteration updates actually produce identical iterates~\citep{vanseijen2015deeper,parr2008analysis}, though it is unknown whether this relationship holds under non-linear function approximation. Therefore, exploring the theoretical bounds on the optimal performance of offline model-based RL methods under non-linear function approximation, as compared to offline dynamic programming methods, remains an open problem.

\section{Applications and Evaluation}
\label{sec:applications}

In this section, we survey and discuss evaluation methods, benchmarks, and applications for offline reinforcement learning methods. As discussed in Section~\ref{sec:introduction}, and as we will discuss further in Section~\ref{sec:discussion}, it is very likely that the full potential of offline reinforcement learning methods has yet to be fully realized, and perhaps the most exciting applications of such methods are still ahead of us. Nonetheless, a number of prior works have applied offline reinforcement learning in a range of challenging domains, from safety-critical real-world physical systems to large-scale learning from logged data for recommender systems. We first discuss how offline reinforcement learning algorithms have been evaluated in prior work, and then discuss specific application domains where such methods have already made an impact.

\subsection{Evaluation and Benchmarks}

While individual application domains, such as recommender systems and healthcare, discussed below, have developed particular domain-specific evaluations, the general state of benchmarking for modern offline reinforcement learning research is less well established. In the absence of well-developed evaluation protocols, one approach employed in recent work is to utilize training
data collected via a standard online reinforcement learning algorithm, using either the entire replay buffer for an off-policy algorithm for training~\citep{kumar2019stabilizing,agarwal2019optimistic,fujimoto2018off}, or even data from the optimal policy. However, this evaluation setting is
rather unrealistic, since the entire point of utilizing offline reinforcement learning algorithms in the real world is to obtain a policy that is better than the best behavior in the dataset, potentially in settings where running reinforcement learning online is impractical due to cost or safety concerns. A simple compromise solution is to utilize data from a ``suboptimal'' online reinforcement learning run, for example by stopping the online process early, saving out the buffer, and using this buffer as the dataset for offline RL~\citep{kumar2019stabilizing}. However, even this formulation does not fully evaluate capabilities of offline reinforcement learning methods, and the statistics of the training data have a considerable effect on the difficult of offline RL~\citep{d4rl}, including how concentrated the data distribution is around a specific set of trajectories, and how multi-modal the data is. Broader data distributions (i.e., ones where $\behavior(\ba|\bs)$ has higher entropy) are generally easier to learn from, since there are fewer out-of-distribution actions. On the other hand, highly multi-modal behavior policies can be extremely difficult to learn from for methods that require explicit estimation of $\behavior(\ba|\bs)$, as discussed in Section~\ref{sec:policy_constraints} and \ref{sec:adp_challenges}. Our recently proposed set of offline reinforcement learning benchmarks aims to provide standardized datasets and simulations that cover such difficult cases~\citep{d4rl}.

\begin{wrapfigure}{r}{0.39\textwidth} 
\vspace{-19pt}
\centering
\includegraphics[width=0.33\textwidth]{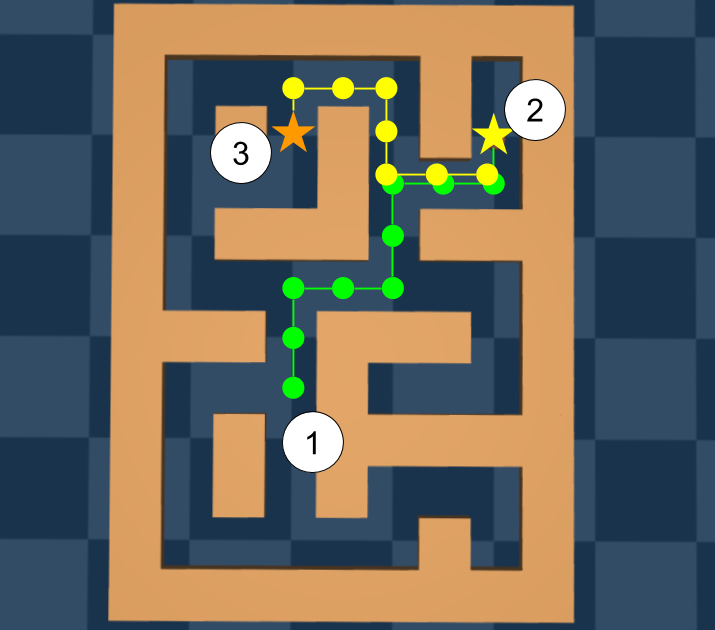}
\caption{\label{fig:stitch} An example of exploiting compositional structure in trajectories to find shortest paths in the Maze2D environment in the D4RL benchmark suite~\citep{d4rl}. }
\vspace{-28pt}
\end{wrapfigure}
A reasonable question we might ask in regard to datasets for offline RL is: in which situations might we actually expect offline RL to yield a policy that is significantly better than \emph{any} trajectory in the training set? While we cannot expect offline RL to discover actions that are better than any action illustrated in the data, we can expect it to effectively utilize the compositional structure inherent in any temporal process. This idea is illustrated in Figure~\ref{fig:stitch}: if the dataset contains a subsequence illustrating a way to arrive at state $2$ from state $1$, as well as a separate subsequence illustrating how to arrive at state $3$ from state $2$, then an effective offline RL method should be able to learn how to arrive at state $3$ from state $1$, which might provide for a substantially higher final reward than any of the subsequences in the dataset. When we also consider the capacity of neural networks to generalize, we could imagine this sort of ``transitive induction'' taking place on a \emph{portion} of the state variables, effectively inferring potentially optimal behavior from highly suboptimal components. This capability can be evaluated with benchmarks that explicitly provide data containing this structure, and the D4RL benchmark suite provides a range of tasks that exercise this capability~\citep{d4rl}.

Accurately evaluating the performance of offline RL algorithms can be difficult, because we are typically interested in maximizing the \textit{online} performance of an algorithm. When simulators are available, online evaluations can be cheaply performed within the simulator order to benchmark the performance of algorithms. Off-policy evaluation (OPE) can also be used to estimate the performance of policies without explicit online interaction, but it is an active area of research as discussed in Section~\ref{sec:ope}. Nevertheless, OPE is a popular tool in areas such as online advertising~\citep{li2010contextual} or healthcare~\citep{murphy2001marginal} where online evaluation can have significant financial or safety consequences. In certain domains, human experts can be used to assess the quality of the decision-making system. For example,~\citet{jaques2019way} uses crowd-sourced human labeling to judge whether dialogue generated by an offline RL agent is fluent and amicable, and~\citet{raghu2017deep} evaluates using a qualitative analysis based one domain experts for sepsis treatment.

\subsection{Applications in Robotics}

Robotics is an appealing application domain for offline reinforcement learning, since RL has the potential to automate the acquisition of complex behavioral skills for robots -- particularly with raw sensory observations, such as camera images -- but conducting online data collection for each robotic control policy can be expensive and impractical. This is especially true for robots that must learn to act intelligently in complex open-world environments, since the challenge of robust visual perception alone already necessitates large training sets. The ImageNet Large-Scale Visual Recognition
Challenge~\citep{russakovsky2015imagenet} stipulates a training set of 1.5 million images for open-world object recognition, and it seems reasonable that the sample complexity for a robotic RL algorithm that must act in similar real-world settings should be at least of comparable size. For this reason, utilizing previous collected data can be of critical importance in robotics.

Several prior works have explored offline RL methods for learning robotic grasping, which is a particularly appealing task offline RL methods, since it requires the ability to generalize to a wide range of 
\begin{wrapfigure}{r}{0.5\textwidth} 
\centering
\includegraphics[width=0.5\textwidth]{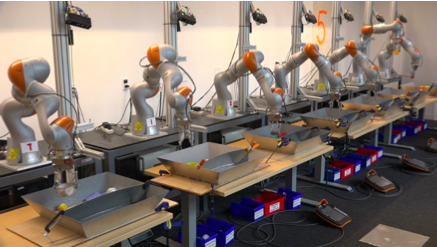}
\caption{Large-scale robotic grasping data collection. \citet{kalashnikov2018qtopt} describes how a dataset of over 500,000 grasp trials collected from multiple robots was used to train a vision-based grasping policy, comparing fully offline training and online fine-tuning.\label{fig:qtopt}}
\vspace{-15pt}
\end{wrapfigure}
objects~\citep{pinto2016supersizing,levine2018learning,kalashnikov2018qtopt,zeng2018learning}. Such methods have utilized approximate dynamic programming~\citep{kalashnikov2018qtopt} (see Figure~\ref{fig:qtopt}), as well as more domain-specific algorithms, such as a single-step bandit formulation~\citep{pinto2016supersizing}. Outside of robotic grasping, \citet{ebert2018visual} propose a model-based algorithm based on prediction of future video frames for learning a variety of robotic manipulation skills from offline data, while \citet{dasari2019robonet} expand on this approach with a large and diverse dataset of robotic interaction.
\citet{cabi2019framework} propose to use reward learning from human preferences combined with offline RL to provide a user-friendly method for controlling robots for object manipulation tasks. In the domain of robotic navigation, \citet{Mo18AdobeIndoorNav} propose a dataset of visual indoor scenes for reinforcement learning, collected via a camera-mounted-robot, and \citet{mandlekar2020learning} proposes a dataset of human demonstrations for robotic manipulation. \citet{kahn2020badgr} discuss how a method based on prediction of future reward signals, blending elements of model-based and model-free learning, can learn effective navigation policies from data collected in the real world using a random exploration policy. Pure model-based methods in robotics typically involve training a model on real or simulated data, and then planning within the model to produce a policy that is executed on a real system. Approaches have included using Gaussian process models for controlling blimps~\citep{ko2007gaussian}, and using linear regression~\citep{koppejan2009neuroevolutionary} and locally-weighted Bayesian regression~\citep{bagnell2001autonomous} for helicopter control.

\subsection{Applications in Healthcare}

Using offline reinforcement learning in healthcare poses several unique challenges~\citep{gottesman2018evaluating}. Safety is a major concern, and largely precludes any possibility of online exploration. Datasets can also be significantly biased towards serious outcomes~\citep{gottesman2019guidelines}, since minor cases rarely require treatment, and can lead na\"{i}ve agents to erroneous conclusions, for example that any drug treatment may cause death simply because it is not prescribed to otherwise healthy individuals. 

The MIMIC-III dataset~\citep{johnson2016mimic}, which contains approximately 60K medical records from ICUs, has been influential in enabling data-driven research in healthcare treatment. Q-learning methods on this dataset has been applied to problems such as the treatment of sepsis~\citep{raghu2017deep} and optimizing the use of ventilators~\citep{prasad2017reinforcement}.~\citet{wang2018supervised} apply actor-critic methods on MIMIC-III to determine drug recommendations.  

\begin{wrapfigure}{l}{0.4\textwidth}
\vspace{-10pt}
\centering
\includegraphics[width=0.4\textwidth]{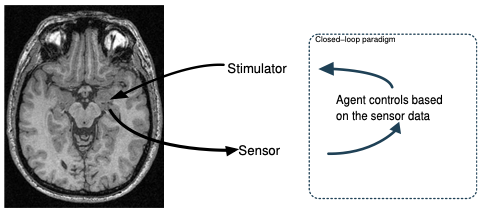}
\caption{A real-time epilepsy treatment system, train using offline reinforcement learning~\citep{guez2008adaptive}.}
\vspace{-10pt}
\end{wrapfigure} 

Outside of ICU settings, offline RL applications include learning from recordings of seizure activity in mouse brains in order to determine optimal stimulation frequencies for reducing epileptic seizures~\citep{guez2008adaptive}.
Offline RL has also been used for optimizing long term treatment plans.~\citet{shortreed2011informing} uses offline fitted Q-iteration for optimizing schizophrenia treatment,~\citet{nie2019learning} uses doubly-robust estimators to safely determine proper timings of medical treatments, and~\citet{tseng2017deep} uses a model-based approach for lung cancer treatment. Careful application of offline RL that can handle such challenges may offer healthcare providers powerful assistive tools for optimizing the care of patients and ultimately improving outcomes.

\subsection{Applications in Autonomous Driving}

As in healthcare, a significant barrier to applying reinforcement learning in the domain of self-driving vehicles is safety. 
In the online setting, exploratory agents can select actions that lead to catastrophic failure, potentially endangering the lives of the passengers. Thus, offline RL is potentially a promising tool for enabling, safe, effective learning in autonomous driving.

\begin{wrapfigure}{r}{0.4\textwidth} 
\vspace{-10pt}
\centering
\includegraphics[width=0.4\textwidth]{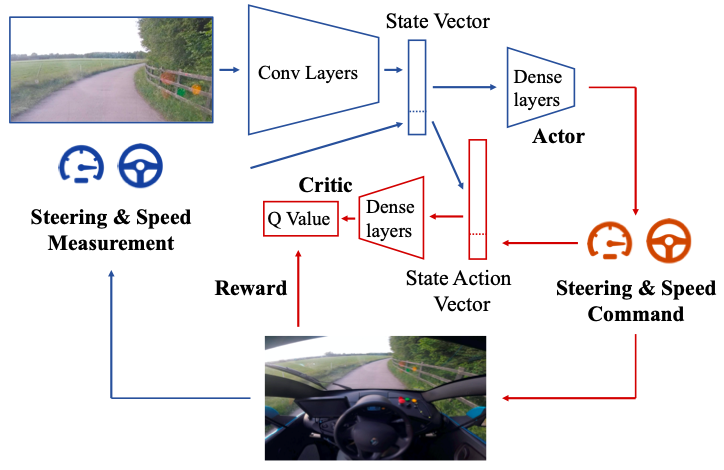}
\caption{A road following system trained end-to-end via reinforcement learning~\citep{kendall2019learning}. \label{fig:kendall}}
\vspace{-10pt}
\end{wrapfigure}
While offline RL has not yet found significant application in actual real-world self-driving vehicles~\citep{yurtsever2020survey}, learning-based approaches have been gaining in popularity.
RobotCar~\citep{maddern2017robotcar} and BDD-100K~\citep{yu2018bdd100k} are both large video datasets containing thousands of hours of real-life driving activity. Imitation learning has been a popular approach
towards end-to-end, data-driven methods in autonomous driving~\citep{bojarski2016end,sun2018fast,pan2017agile,codevilla2018end}. Reinforcement learning approaches have been applied in both simulation~\citep{sallab2017deep} and in the real world, with human interventions in case the vehicle violates a safety constraint~\citep{kendall2019learning} (see Figure~\ref{fig:kendall}).
Model-based reinforcement learning methods that employ constraints to keep the agent close to the training data for the model, so as to avoid out-of-distribution inputs as discussed in Section~\ref{sec:model_based}, can effectively provide elements of imitation learning when training on driving demonstration data, as for example in the case of deep imitative models (DIMs)~\citep{rhinehart2018deep}. Indeed, with the widespread availability of high-quality demonstration data, it is likely that effective methods for offline RL in the field of autonomous driving will, explicitly or implicitly, combine elements of imitation learning and reinforcement learning.

\subsection{Applications in Advertising and Recommender Systems}

Recommender systems and advertising are particularly suitable domains for offline RL because data collection is easy and efficient, and can be obtained by logging user behavior. However, these domains are also ``safety critical,'' in the sense that a highly suboptimal policy may result in large monetary losses, thereby making unconstrained online exploration problematic. Thus, offline RL approaches have a long history of application in this area. 

Off-policy evaluation is commonly used as a tool for performing A/B tests and estimating the performance of advertising and recommender systems without additional interaction with the environment. In contrast to the other applications discussed, policy evaluation for recommender systems is typically formulated within a contextual bandit problem~\citep{langford2008exploration,li2010contextual}, where states may correspond to user history and actions correspond to recommendations. This approximation removes the need for \textit{sequential} decision making, but can introduce approximation errors if actions have temporal dependence as in domains such as robotics or healthcare. 

Applications of offline RL for recommender systems include slate and whole-page optimization~\citep{swaminathan2017off}, applying doubly robust estimation to estimate website visits~\citet{dudik2014doubly}, and A/B testing for click optimization~\citep{gilotte2018offline}. Policy learning from logged, offline data has included studies on optimizing newspaper article click-through-rates~\citep{strehl2010learning,garcin2014offline}, advertisement ranking on search pages~\citep{bottou2013counterfactual}, and personalized ad recommendation for digital marketing~\citep{theocharous2015personalized,thomas2017predictive}.

\subsection{Applications in Language and Dialogue}

Interaction via natural language is not typically thought of as a reinforcement learning problem, but in fact the formalism of sequential decision making can provide a powerful tool for natural language interaction: when dialogue is modeled as a \emph{purposeful} interaction, the RL framework can in principle offer an effective mechanism for learning policies for outputting natural language responses to human interlocutors. The most direct way to utilize standard online RL methods for natural language interaction -- by having machines engage in dialogue with real humans -- can be exceedingly tedious, especially in the early stages of training, when the policy would produce mostly non-sensical dialogue. 
For this reason, offline RL offers a natural avenue to combine the optimal decision making formalism of RL with the kinds of large datasets of human-to-human conversations available in NLP.

In prior work, offline RL approaches have been applied in the areas of dialogue and language interfaces, where datasets consist of logged interactions, such as agent-customer transcripts~\citep{zhou2017end}. An example of an application is dialogue management, which is typically concerned 
\begin{wrapfigure}{l}{0.7\textwidth} 
\vspace{-10pt}
\centering
\includegraphics[width=0.7\textwidth]{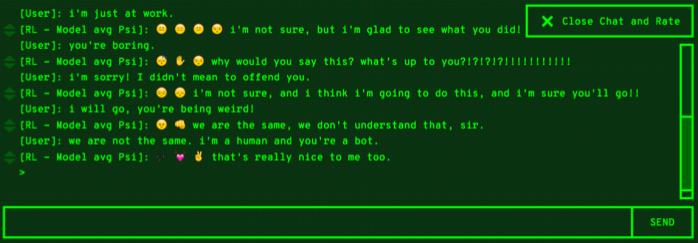}
\caption{A dialogue agent trained via offline reinforcement learning interacting with a human user, with the aim of elicit responses with positive sentiment~\citep{jaques2019way}. \label{fig:dialogue}}
\vspace{-10pt}
\end{wrapfigure}
with accomplishing a specific goal, such as retrieving information. 
Examples of this have included applications of offline RL to the problem of flight booking~\citep{henderson2008hybrid}, restaurant information retrieval~\citet{pietquin2011sample}, and restaurant recommendation~\citep{kandasamy2017batch}. \citet{jaques2019way} applies offline RL to the problem of dialogue generation, and focuses on producing natural responses that elicit positive feedback from human users. An example of an interaction with a trained agent is shown in Figure~\ref{fig:dialogue}.


\section{Discussion and Perspectives}
\label{sec:discussion}

Offline reinforcement learning offers the possibility of turning reinforcement learning -- which is conventionally viewed as a fundamentally active learning paradigm -- into a data-driven discipline, such that it can benefit from the same kind of ``blessing of scale'' that has proven so effective across a range of supervised learning application areas~\citep{lecun2015deep}. However, making this possible will require new innovations that bring to bear sophisticated statistical methods and combine them with the fundamentals of sequential decision making that are conventionally studied in reinforcement learning. Standard off-policy reinforcement learning algorithms have conventionally focused on dynamic programming methods that can utilize off-policy data, as discussed in Section~\ref{sec:rl_prelims} and Section~\ref{sec:adp}, and importance sampling methods that can incorporate samples from different sampling distributions, as discussed in Section~\ref{sec:offline_pg}.

However, both of these classes of approaches struggle when scaled up to complex high-dimensional function approximators, such as deep neural networks, high-dimensional state or observation spaces, and temporally extended tasks. As a result, the standard off-policy training methods in these two categories have generally proven unsuitable for the kinds of complex domains typically studied in modern deep reinforcement learning. More recently, a number of improvements for offline RL methods have been proposed that take into account the statistics of distributional shift, via either policy constraints or uncertainty estimation, such as the policy constraint formulations that we discuss in Section~\ref{sec:policy_constraints}. These formulations have the potential to mitigate the shortcomings of early methods, by explicitly account for the key challenge in offline RL: distributional shift due to differences between the learned policy and the behavior policy.

More generally, such methods shed light on the fact that offline RL is, at its core, a counter-factual inference problem: given data that resulted from a given set of decisions, infer the consequence of a \emph{different} set of decisions. Such problems are known to be exceptionally challenging in machine learning, because they require us to step outside of the commonly used i.i.d. framework, which assumes that test-time queries involve the same distribution as the one that produced the training data~\citep{scholkopf2019causality}. It therefore stands to reason that the initial solutions to this problem, proposed in recent work, should aim to reduce distributional shift, either by constraining the policy's deviation from the data, or by estimating (epistemic) uncertainty as a measure of distributional shift. Moving forward, we might expect that a variety of tools developed for addressing distributional shift and facilitating generalization may find use in offline RL algorithms, including techniques from causal inference~\citep{scholkopf2019causality}, uncertainty estimation~\citep{gal2016dropout,kendall2017uncertainties}, density estimation and generative modeling~\citep{kingma2014semi}, distributional robustness~\citep{sinha2017certifying,sagawa2019distributionally} and invariance~\citep{arjovsky2019invariant}. More broadly, methods that aim to estimate and address distributional shift, constrain distributions (e.g., various forms of trust regions), and evaluate distribution support from samples are all potentially relevant to developing improved methods for offline reinforcement learning.


The counter-factual inference perspective becomes especially clear when we consider model-based offline RL algorithms, as discussed briefly in Section~\ref{sec:model_based}. In this case, the model answers the question: ``what would the resulting state be if the agent took an action other than the one in the dataset?'' Of course, the model suffers from distributional shift in much the same way as the value function, since out-of-distribution state-action tuples can result in inaccurate prediction. Nonetheless, prior work has demonstrated good results with model-based methods, particularly in regard to generalization with real-world data~\citep{finn2017deep,ebert2018visual}, and a range of works on predicting physical phenomena have utilized offline datasets~\citep{lerer2016learning,battaglia2016interaction}. This suggests that model learning may be an important component of effective future offline reinforcement learning methods.

To conclude our discussion of offline reinforcement learning, we will leave the reader with the following thought. While the machine learning community frequently places considerable value on design of novel algorithms and theory, much of the amazing practical progress that we have witnessed over the past decade has arguably been driven just as much by advances in datasets as by advances in methods. Indeed, widely deployed techniques in computer vision and NLP utilize learning methods that are relatively old and well understood, and although improvements in architectures and models have driven rapid increase in performance, the increasing size and diversity of datasets -- particularly in real-world applications -- have been an instrumental driver of progress. In real-world applications, collecting large, diverse, representative, and well-labeled datasets is often far more important than utilizing the most advanced methods. In the standard active setting in which most reinforcement learning methods operate, collecting large and diverse datasets is often impractical, and in many applications, including safety-critical domains such as driving, and human-interactive domains such as dialogue systems, it is prohibitively costly in terms of time, money, and safety considerations. Therefore, developing a new generation of \emph{data-driven} reinforcement learning may usher in a new era of progress in reinforcement learning, both by making it possible to bring it to bear on a range of real-world problems that have previously been unsuited for such methods, and by enabling current applications (e.g., in robotics or autonomous driving) to benefit from much larger, more diverse, and more representative datasets that can be reused effectively across experiments.

\bibliographystyle{apalike}
\bibliography{references}

\end{document}